%% file: main.tex
\documentclass[twoside]{article}

\usepackage[accepted]{aistats2023}
\usepackage[authoryear, round]{natbib}

\usepackage{hyperref}
\usepackage{graphicx}
\input{math_commands.tex}

\usepackage{mathtools}
\usepackage{amsthm,amssymb,amsmath,amsfonts}
\usepackage{graphicx}
\usepackage{subfigure}

\newtheorem{deft}{Definition}
\newtheorem{lem}{Lemma} 
\newtheorem{thm}{Theorem}

\newtheorem{assum}{Assumption}
\newtheorem{cond}{Condition}

\begin{document}



\twocolumn[

\aistatstitle{Global Convergence of Over-parameterized Deep Equilibrium Models}

\aistatsauthor{ Zenan Ling\textsuperscript{* 1,2} \And Xingyu Xie\textsuperscript{1} \And Qiuhao Wang\textsuperscript{3} \And Zongpeng Zhang\textsuperscript{3} \And Zhouchen Lin\textsuperscript{\dag 1,4,5}}
\aistatsaddress{  \textsuperscript{1}National Key Lab. of General Artificial Intelligence, School of Intelligence Science and Technology, Peking University\\ \textsuperscript{2}EIC, Huazhong University of Science and Technology \\ \textsuperscript{3}Center for Data Science, Academy for Advanced Interdisciplinary Studies, Peking University\\ \textsuperscript{4}Institute for Artificial Intelligence, Peking University\\ \textsuperscript{5}Peng Cheng Laboratory }  ]

\begin{abstract}
  A deep equilibrium model (DEQ) is  implicitly defined through an equilibrium point of an infinite-depth  weight-tied model with an input-injection. Instead of  infinite computations, it solves an equilibrium point directly with root-finding and computes gradients with implicit differentiation. In this paper, the training dynamics of over-parameterized DEQs are investigated, and we propose a novel probabilistic framework to overcome the challenge arising from the weight-sharing and the infinite depth.  By supposing a condition on the initial equilibrium point, we prove that the gradient descent  converges to a globally optimal solution at a linear convergence rate for the quadratic loss function.  We further perform a fine-grained non-asymptotic analysis about random DEQs and the corresponding weight-untied
models, and show that  the required initial condition is satisfied via mild over-parameterization. Moreover, we  show that the unique equilibrium point always exists during the training.
\end{abstract}

\section{INTRODUCTION}
\input{intro}
\section{PRELIMINARIES}
\input{preliminaries}
\section{MAIN RESULTS}\label{sec:main}
\input{result}

\section{ANALYSIS AT INITIALIZATION }\label{sec:initial}
\input{initialization}
\section{NUMERICAL EXPERIMENTS}
\input{exp}
\section{CONCLUSION}
\input{conclusion}

\bibliography{main}
\bibliographystyle{apalike}  
% \clearpage
\appendix
\input{appendix}

\end{document}

%% file: math_commands.tex
\usepackage{amsmath,amsfonts,bm, physics}

\def\eqref#1{Eq.~(\ref{#1})}
\def\1{\bm{1}}

\def\ru{{\textnormal{u}}}
\def\rv{{\textnormal{v}}}

\def\rvw{{\mathbf{w}}}

\def\va{{\bm{a}}}
\def\vb{{\bm{b}}}

\def\vg{{\bm{g}}}
\def\vh{{\bm{h}}}

\def\vp{{\bm{p}}}
\def\vq{{\bm{q}}}

\def\vv{{\bm{v}}}

\def\vx{{\bm{x}}}
\def\vy{{\bm{y}}}
\def\vz{{\bm{z}}}

\def\mA{{\bm{A}}}
\def\mB{{\bm{B}}}

\def\mD{{\bm{D}}}

\def\mG{{\bm{G}}}
\def\mH{{\bm{H}}}
\def\mI{{\bm{I}}}
\def\mJ{{\bm{J}}}
\def\mK{{\bm{K}}}

\def\mM{{\bm{M}}}

\def\mR{{\bm{R}}}

\def\mU{{\bm{U}}}
\def\mV{{\bm{V}}}
\def\mW{{\bm{W}}}
\def\mX{{\bm{X}}}

\def\mZ{{\bm{Z}}}

\DeclareMathAlphabet{\mathsfit}{\encodingdefault}{\sfdefault}{m}{sl}
\SetMathAlphabet{\mathsfit}{bold}{\encodingdefault}{\sfdefault}{bx}{n}

\def\gN{{\mathcal{N}}}

\def\*#1{\mathbf{#1}}
\def\$#1{\mathcal{#1}}
\def\^#1{\mathbb{#1}}

\newcommand{\E}{\mathbb{E}}

%% file: Intro.tex
Deep equilibrium models (DEQs)~\citep{NEURIPS2019_01386bd6} have recently emerged as a new neural network design paradigm. A DEQ is equivalent to  an infinite-depth weight-tied model with input-injection. Different from conventional (explicit) neural networks,  DEQs generate features by directly solving equilibrium points of implicit equations.  
DEQs also have the remarkable advantage that the gradients  can be computed analytically by backpropagation only through the equilibrium point with implicit differentiation. Therefore, training a DEQ only requires constant memory.  

DEQs have  achieved impressive performance in various applications such as computer vision~\citep{bai2020multiscale,xie2022optimization}, natural language processing~\citep{NEURIPS2019_01386bd6}, and inverse problems~\citep{gilton2021deep}. 
Although the empirical success of DEQs has been observed  in many recent studies,   theoretical understandings of DEQs are  still limited compared to conventional models. 
In this paper,  we aim to  establish the global convergence of the gradient descent (GD)  associated with  an over-parameterized DEQ, as a step towards understanding general DEQs.

A large body of work~\citep{arora2019exact,du2019gradient,li2018learning} has validated  the effectiveness of over-parameterization  in optimizing feedforward neural networks. The main idea  is to investigate the property at initialization and bound the traveling distance of GD from the
initialization~\citep{2021On}. However, it remains unclear whether these results  can be directly
applied to DEQs.  The implicit weight-sharing is the key challenge. Most standard concentration tools used in previous studies  fail in DEQs. This is because these analyses rely on the independence of initial random weights and features, which is no longer the case in DEQs.
Moreover, it remains unknown whether  over-parameterization is sufficient to guarantee the well-posedness~\citep{winston2020monotone,revay2020lipschitz} of the implicit mapping of a DEQ, which is crucial to the stability of the training process, \emph{e.g.}~\citep{kawaguchi2020theory} uses an extra softmax layer to resolve the well-posedness issue and achieves a global linear convergent rate. However, this result holds only for linear DEQs and it is difficult to be extended to nonlinear activations.

We start  with the gradient analysis. We observe that, in the case of DEQs, the least singular value of the equilibrium points plays a key role in the gradient dynamic. Specifically, if the least singular values of the equilibrium points  at all iterations can be lower bounded by a positive constant, then one can establish a corresponding version of the Polyak-Lojasiewicz  inequality for DEQs, and thus the global convergence of GD can be obtained.

Our main results are based on the following observations. Firstly,  we prove the global convergence of GD by supposing an initial condition on the lower bound of the least singular value of the initial equilibrium points.
The perturbation of the weight matrices is small enough to ensure that the  Lipschitz constant of  the implicit layer transformation is smaller than $1$. This means that, the unique equilibrium point always exists throughout the training.
Our second observation is that,  the required  initial condition holds with a high probability for general Gaussian initialization.  Note that the weight of a DEQ is implicitly re-used across layers. Thus, standard concentration inequalities, which are based on the independence of weight matrices and features, cannot be employed directly in this scenario of random DEQs analysis.
In order to to overcome the technical difficulty, we propose a novel probabilistic framework to approximate the empirical Gram matrix of the equilibrium point with a population Gram matrix induced from a  weight-untied random network with infinite depth.

\subsection{Related Work}
\paragraph{Finite-depth Over-parameterized Feedforward Networks.}Recently, over-parameterization has attracted much attention due to its effectiveness in optimizing finite-depth neural networks.
\citet{jacot2018neural} show that, for smooth activation and infinite wide neural networks, the trajectory of the gradient descent (GD) method could be well-captured by a kernel called the neural tangent kernel (NTK).
For a finite-width feed-forward neural network with smooth activation, \citet{arora2019exact,du2019gradient,li2018learning} prove that the neural networks' dynamics are strongly related to a Gram matrix.
For a finite-width feed-forward neural network with ReLU activation, 
\citet{allen2019convergence,zou2019improved,zou2020gradient,oymak2020toward,nguyen2020global}  estimate the changes of the activation patterns, and show that GD can converge to a global minimum  despite the non-smoothness of activation and non-convexity of the objective function. The only condition that needs to satisfy is that the width of each layer is a polynomial of the number of training samples and the number of layers. In particular, \citet{2021On} provides an alternative framework that only requires tracking the evolution of the last hidden layer rather than the activation pattern. In all 
previous works, a non-asymptotic  analysis at initialization  plays a fundamental role.
Their analysis on the random initialization  relies on the independence between the initial random weights  and features. However, the weight matrix is implicitly shared in a DEQ. Thus,  previous results do not directly apply to DEQs.

\paragraph{Finite-depth Over-parameterized Weight-tied Neural Networks.} For over-parameterized weight-tied models,  ~\citet{yang2019wide,yang2020tensor2,yang2020tensor3}   investigate NTKs of the recurrent neural networks (RNNs) with infinite width by leveraging the ``Gaussian conditioning trick''~\citep{bolthausen2014iterative,bayati2011dynamics}. The similar technique is also used in the study on infinitely wide weight-tied autoencoders~\citep{li2018random}. These works reveal that  infinitely wide Gaussian weight-tied neural networks are essentially Gaussian processes. However, their results do not apply to the regime of finite width. \citet{allen2019convergence-rnn,wang2021provable} show that, for a RNN with finite width,
GD  can converge to a global minimum if  the width of each layer is a polynomial of the number of training samples and the number of layers. However, their results do not apply to DEQs. This is because most upper bounds for the norm of the hidden units fail as  the depth approaches infinity. Thus, a new analysis framework for the convergence of the implicit models is urgently needed.
\paragraph{Over-parameterized DEQs.}  The investigation of over-parameterized DEQs is still in the initial stage. ~\citet{feng2020neural} consider the  NTK  of DEQs with infinite width. This study claims that the NTK of DEQs is equivalent to the corresponding weight-untied models in the regime of infinite width. This study also shows that DEQs have non-degenerate NTKs even in the infinite depth. This observation is similar with ours. However, the analysis on  finite-width DEQs  is not studied  by~\citet{feng2020neural}. To conduct non-asymptotic analysis on finite-width DEQs, one of the key challenges is to estimate the least singular value of the  equilibrium points. To the best of our knowledge, this problem has not been addressed in general settings. Two very recent works~\citep{2021A} and~\citep{gao2022} consider simple cases in which the problem can be transferred into the estimation of the least singular value of  features obtained from a single layer explicit network. Their results hold for  specific predictions or special initialization.   In contrast, we propose a novel probabilistic framework to estimate the least singular value of the equilibrium points in general settings.  Please see  detailed comparisons in the discussion of Section~\ref{sec:initial}.

\subsection{Contributions}
Our contributions are summarized as follows. 
\begin{enumerate}
    \item[(1)] We propose a novel probabilistic framework to analyze DEQs. Our framework addresses the technical challenges arising from the  weight-sharing and the infinite depth. To the best of our knowledge,  this is the first time a fine-grained non-asymptotic analysis on DEQs has been performed in a general setting. 
    
	\item[(2)] We analyze the
	gradient dynamics of DEQs with the quadratic loss function.
	Under an initial condition on the least singular value of the initial equilibrium points, we prove that the gradient descent converges to a global optimum at a linear rate. Based on our initial analysis, we show that the required  initial condition  is satisfied via mild over-parameterization.
	
	\item[(3)] We show that the unique equilibrium point of  an over-parameterized DEQ always exists throughout the training process, even without using any normalization or re-parameterization method.

\end{enumerate}

%% file: preliminaries.tex
\paragraph{Notations.}
We use $\gN(0,\mI)$ to denote the standard Gaussian distribution. We let $[n]\triangleq [1,\cdots,n]$. For a vector $\vv$, $\|\vv\|_2$ is the Euclidean norm of $\vv$. For a matrix $\mA$, we use $\mA_{ij}$ denote its $(i,j)$-th entry. We use $\left\|\mA\right\|_F$ to  denote the Frobenius norm and $\left\|\mA\right\|_2$ to denote the  operator norm. If a matrix is positive semi-definite, we use $\lambda_{\min}(\mA)$ and $\sigma_{\min}(\mA)$ to denote its least eigenvalue and singular value, respectively.
We let $\order{\cdot}$, $\Theta(\cdot)$ and $\Omega(\cdot)$ denote standard Big-O, Big-Theta, and Big-Omega notations, respectively.
We use $\phi(\cdot)$ to denote the ReLU function, namely $\phi(x) = \max(x,0)$.

\subsection{Problem Setup}
	We define a vanilla deep equilibrium model (DEQ) with the transform at the $l$-th layer as
	\begin{equation}
	\mZ^{(l)} = \phi(\mW \mZ^{(l-1)}+\mU \mX ),  
	\label{eq:deql}
	\end{equation}
	where  $\mX=[\vx_1,\cdots,\vx_n]\in \mathbb{R}^{d \times n}$ denotes the training inputs,  $\mU \in \mathbb{R}^{m \times d}$ and $\mW \in \mathbb{R}^{m \times m}$ are trainable weight matrices,   and $\mZ^{(l)} \in \mathbb{R}^{m \times n}$ is the output feature at the $l$-th hidden layer. The output  of the last hidden layer is defined by $ \mZ^* \triangleq \lim_{l \rightarrow \infty} \mZ^{(l)}$. Therefore, instead of
running the infinitely deep layer-by-layer forward propagation, $\mZ^*$ can be calculated by directly solving the equilibrium point of the following equation 
	\begin{equation}
	   \mZ^{*} = \phi(\mW \mZ^{*}+\mU \mX ). 
	   \label{eq:deq}
	\end{equation}
	
Let $\vy=[y_1,\cdots, y_n]\in \mathbb{R}^{ n}$ denote the labels, and $\hat\vy(\boldsymbol{\theta}) = \va^\top\mZ^*$ be the prediction function with  $\va \in \mathbb{R}^{m}$ being a trainable vector and $\boldsymbol{\theta}=\text{vec}\left(\mW,\mU,\va\right)$. The object of our interest is the empirical risk minimization
problem with the quadratic loss function \begin{equation*}
    \Phi(\boldsymbol{\theta})=\frac{1}{2} \left\|\hat{\vy}(\boldsymbol{\theta}) -\vy\right\|_2^2.
\end{equation*}
To do so, we consider the gradient descent (GD) update $\boldsymbol{\theta}_{\tau+1}=\boldsymbol{\theta}_{\tau}-\eta\nabla\Phi\left(\boldsymbol{\theta}_{\tau}\right)$, where $\eta$ is the learning rate and $\boldsymbol{\theta}_{\tau}=\operatorname{vec}\left(\mW(\tau),\mU(\tau),\va(\tau)\right)$  is the parameter we optimize over at step $\tau$.  For  notational simplicity, we  omit the superscribe and  denote $\mZ$ to be the equilibrium $\mZ^{*} $ when it is clear from the context. Moreover,
the Gram matrix of the equilibrium point is defined by $\mG(\tau) \triangleq \mZ(\tau)^\top\mZ(\tau)$ and we denote its least eigenvalue as $\lambda_\tau = \lambda_{\min}\left(\mG(\tau)\right)$.

In this paper, we make the following assumptions on the random initialization and  the input data.
\begin{assum}[Random initialization] Take $\sigma_w^2<1/8$.
We assume that $\mW$ is initialized with an $m\times m$ matrix with \emph{i.i.d.} entries $\mW_{ij}\sim\gN(0,2\sigma_w^2/m)$, $\mU$ is initialized with an $m\times d$ matrix with \emph{i.i.d.} entries $\mU_{ij}\sim\gN(0,2/d)$, $\va$ is initialized with a random vector with \emph{i.i.d.} entries $\va\sim\gN(0, 1/m)$.
\label{assum:initial}
\end{assum}

\begin{assum}[Input data]
We assume that (\romannumeral1) $\left\|\vx_i\right\|_2=\sqrt{d}$, for all $i\in[n]$, and  $\vx_i\nparallel \vx_j$, for each pair $i\neq j \in [n]$, (\romannumeral2) the labels satisfy $|y_i|=\order{1}$ for all $i\in [n]$.
\label{assum:data}
\end{assum}

\subsection{Well-Posedness and Gradients}
\paragraph{Well-Posedness.} For the  stability of the  training of the DEQs, it is crucial to guarantee the existence and uniqueness of the equilibrium points~\citep{winston2020monotone,el2021implicit}. It is equivalent to guarantee the well-posedness of the transformation defined in \eqref{eq:deql}.  In order to ensure the well-posedness, it suffices to take $\|\mW(\tau)\|_2<1$, for all $\tau\geq0$, with which  \eqref{eq:deql} becomes a \emph{contractive mapping}.
From Lemma~\ref{lem:initialsigma}, we know that $\|\mW(0)\|_2<1$ holds with a high probability under Assumption~\ref{assum:initial}. Lemma~\ref{lem:initialsigma} is a consequence of   standard  bounds concerning the singular values of Gaussian random matrices~\citep{vershynin2018high}. 
\begin{lem}
Let $\mW$ be an $m \times m$ random matrix with  \emph{i.i.d.} entries $\mW_{ij}\sim \gN(0,\frac{2\sigma_w^2}{m})$. With probability at least $1-\exp{-\Omega(m)}$, it holds that
$\|\mW\|_2\leq 2\sqrt{2}\sigma_w$.
\label{lem:initialsigma}
\end{lem}
In Section~\ref{sec:main}, we  show that the condition of $\|\mW(\tau)\|_2<1$ always holds for $\tau\geq 0$. It is worth mentioning that the  constraint on the spectral norm  of $\mW(\tau)$  can be lightened by that on the spectral radius of $\mW(\tau)$ through special re-parameterization methods~\citep{winston2020monotone,revay2020lipschitz}. However, in this paper,  we do not make extra assumptions on specific   structures of weight matrices, and thus our constraint on the spectral norm of $\mW(\tau)$ is in general  mild. 
\paragraph{Gradients.} The gradients of conventional neural networks are usually computed via backpropagation  through all the intermediate layers. On the contrary, the gradients \emph{w.r.t.}  parameters of a DEQ are computed analytically via backpropagation only through the equilibrium point $\mZ$  by applying   the \emph{implicit function theorem}.
Specifically, note that
the equilibrium
point  of \eqref{eq:deq} is the root of the function \[F(\tau) \triangleq \mZ(\tau)-\phi(\mW(\tau)\mZ(\tau)+ \mU(\tau)\mX(\tau)).\] Let $\mJ(\tau) \triangleq \partial\text{vec}(F(\tau))/\partial\text{vec}(\mZ(\tau))$ denote the Jacobian matrix.
% the gradient of the loss $\Phi(\tau)$\footnote{To simplify the notation, we omit the parameter $\boldsymbol{\theta}$ and write just $\Phi(\tau)$ and $\hat{\vy}(\tau)$.} \emph{w.r.t.} $\boldsymbol{\theta}_\tau$ is given by
% \begin{equation}\nonumber
% \left.\frac{\partial \Phi(\tau)}{\partial \boldsymbol{\theta}_\tau}\right|_{\mZ(\tau)}=\left.\text{vec}\left(\nabla_{\mZ(\tau)} \Phi(\tau)\right) \mJ(\tau)^{-1} \frac{\partial F(\boldsymbol{\theta}_\tau)}{\partial \boldsymbol{\theta}_\tau}\right|_{\mZ(\tau)}.
% \end{equation}
One can derive that \begin{equation*}
    \mJ(\tau) = \mI_{mn} -\mD(\tau)\left(\mI_n\otimes\mW(\tau)\right),
\end{equation*}
where  $\mD(\tau) \triangleq \operatorname{diag}[\operatorname{vec}(\phi'(\mW(\tau)\mZ(\tau)+\mU(\tau)\mX(\tau)))]$\footnote{We use $\mathbb{I}\{x\geq 0\} $ as the (sub)-gradient of ReLU.}.
Using the Lipschitz property of ReLU, it is easy to check that $\mJ(\tau)$ is invertible if $\|\mW(\tau)\|_2<1$.
 The gradient of each trainable parameter is given by the following lemma\footnote{To simplify the notation, we omit the parameter $\boldsymbol{\theta}$ and write just $\Phi(\tau)$ and $\hat{\vy}(\tau)$.}. 
\begin{lem} If $\mJ(\tau)$ is invertible, 
the gradient of the objective function $\Phi(\tau)$ \emph{w.r.t.} each trainable parameters is given by
\begin{equation}
\left\{\begin{array}{l}
\operatorname{vec}\left(\nabla_\mW \Phi(\tau)\right) = \left(\mZ(\tau)\otimes\mI_m\right)\mR(\tau)^\top\left(\hat{\vy}(\tau) - \vy\right)\\
\operatorname{vec}\left(\nabla_\mU \Phi(\tau) \right)= \left(\mX(\tau)\otimes\mI_m\right)\mR(\tau)^\top\left(\hat{\vy}(\tau) - \vy\right)\\
\nabla_\va \Phi(\tau) = \mZ(\tau) \left(\hat{\vy}(\tau) - \vy\right)
\end{array},\right.
\end{equation}
where $\mR(\tau) =  \left(\va(\tau)\otimes\mI_n\right)\mJ(\tau)^{-1}\mD(\tau)$.
\label{lem:grad}
\end{lem}
\subsection{Polyak-Lojasiewicz Inequalities}
 Polyak-Lojasiewicz (PL) inequality~\citep{1963Gradient} is a  commonly used recipe to prove linear convergence of GD algorithms~\citep{nguyen2020global,2021On}.
 In order to obtain a corresponding version of PL inequalities for DEQs,  our starting observation  is that
\begin{equation}
    \left\|\nabla_{\boldsymbol{\theta}} \Phi(\tau)\right\|_2^2\geq2\lambda_{\min}\left(\mH(\tau)\right)\Phi(\tau),
    \label{eq:pl1}
\end{equation}
where  $\mH(\tau) = \mH_1(\tau)+\mH_2(\tau)+\mH_3(\tau)$ is a sum of three positive semi-definite matrices defined as 
\begin{align*}\left\{
\begin{array}{l}
    \mH_1(\tau) = \mG(\tau)\\ \mH_2(\tau)=\mR(\tau)\left(\mG(\tau)\otimes\mI_m\right)\mR(\tau)^\top \\ \mH_3(\tau)=\mR(\tau)\left(\mX^\top\mX\otimes\mI_m\right)\mR(\tau)^\top 
    \end{array},\right.
\end{align*}
\eqref{eq:pl1} is a direct application of Lemma~\ref{lem:grad}.
It suggests that if $\lambda_{\min}\left(\mH(\tau)\right)$ can be lower bounded  away from zero, both at initialization and throughout the training, then one can establish a  PL inequality that holds for the loss function, and thus GD  converges to a global minimum. 
However, it is technically difficult to directly estimate the lower bound of $\lambda_{\min}\left(\mH(\tau)\right)$ because (1) for $\tau=0$, $\mH(0)$ involves both the sum and multiplication of random matrices with complex structures, and (2) for $\tau>0$, $\mH(\tau)$ involves the derivatives of ReLU at activation neurons which requires the estimation related to the  changes of the activation patterns~\citep{2018Gradient,zou2019improved}.

To make the problem tractable, we further observe that  $\lambda_{\min}\left(\mH(\tau)\right)\geq\lambda_\tau$, \emph{i.e.} the least eigenvalue of the Gram matrix of the equilibrium point. Applying this observation to \eqref{eq:pl1}, one obtains
\begin{equation}
    \left\|\nabla_{\boldsymbol{\theta}} \Phi(\tau)\right\|_2^2\geq2\lambda_\tau\Phi(\tau).
\end{equation}
This means that,  in order to obtain a PL-like inequality for DEQs,   it suffices to bound   the changes of $\mG(\tau)$ throughout the training if $\lambda_0$   is  bounded away from zero at  initialization. 
In Section~\ref{sec:main},   we show that it holds $\lambda_\tau\geq\frac{1}{2}\lambda_0$ for every $\tau>0$, and  we have a PL inequality for DEQs: $\left\|\nabla_{\boldsymbol{\theta}} \Phi(\tau)\right\|_2^2\geq\lambda_0\Phi(\tau)$. Based on this, we  prove  that GD  converges to a global optimum at a linear rate. The result of Section~\ref{sec:main} is built upon an initial condition on $\lambda_0$, and  we further demonstrate that such an initial condition can be satisfied with mild over-parameterization.
\subsection{Challenges in initial analysis}
The initial condition on the lower bound of  $\lambda_0$ plays a fundamental role in our convergence result. It is hard to estimate $\lambda_0$ directly. A common way is to build a concentration inequality between the initial \emph{empirical} Gram matrix $\mG$ and the corresponding \emph{population} Gram matrix with easily estimated least eigenvalue. In the case of DEQs, we consider a population Gram matrix  $\mK$  defined as follows.
\begin{deft} We define the 
 population Gram matrices  $\mK^{(l)}$ of each layer  recursively  as  
\begin{equation}
\begin{aligned}
& \mK^{(0)} = 0\\
    &\boldsymbol{\Lambda}_{ij}^{(l)} =  \left[\begin{array}{cc}
 \sigma_w^2\mK_{ii}^{(l-1)} +1 &  \sigma_w^2\mK_{ij}^{(l-1)} +\frac{1}{d}\vx_i^\top\vx_j\\ 
\sigma_w^2\mK_{ji}^{(l-1)} + \frac{1}{d}\vx_j^\top\vx_i &  \sigma_w^2\mK_{jj}^{(l-1)} +1
\end{array}\right],\\
&\mK_{ij}^{(l)} = 2 \E_{(\ru,\rv)^\top\sim \gN(0, \boldsymbol{\Lambda}_{ij}^{(l)})}\left[\phi(\ru)\phi(\rv)\right]
\end{aligned}
\label{eq:defkraw}
\end{equation}
for  $ l \geq 1$ and $(i,j)\in[n] \times [n]$. Letting $l\rightarrow \infty$, we define $\mK \triangleq \mK^{(\infty)}$ and $\lambda_*\triangleq\lambda_{\min}(\mK)$.
% {\footnote{ The convergence of $\mK^{(l)}$ for $l\rightarrow \infty$ and the positive definiteness of $\mK$ are deferred to  Section~\ref{sec:inital1}.}}
\label{def:defs}
\end{deft}
The population Gram matrix  $\mK$ is induced from an \emph{infinite-depth weight-untied} model. The convergence of $\mK^{(l)}$ for $l\rightarrow \infty$ and the positive definiteness of $\mK$ are deferred to  Section~\ref{sec:inital1}. The least eigenvalue of $\mK$ is the fundamental quantity  that determines the lower bound of $\lambda_0$.

Non-asymptotic analysis on  DEQs is more difficult than explicit models. The weight-sharing is the key technical challenge, and one cannot resort to standard concentration tools. Specifically, note that initial $\mG_{ij}$ is implicitly defined as $\vz_i^{\top}\vz_j=\phi([\mW,\mU][\vz_i^\top,\vx_i^\top]^\top)\phi([\mW,\mU][\vz_j^\top,\vx_j^\top])${\footnote{For notational simplicity, we denote  $\mW(0)$, $\mU(0)$ and $\mG(0)$ by  $\mW$ , $\mU$ and $\mG$, respectively.}}. On one hand, one cannot directly apply the standard inequality like previous works. This is because they rely on the independence between initial weight matrices and features, which is no longer the case in DEQs. On the other hand, one cannot directly use the standard $\varepsilon$-net argument~\citep{vershynin2018high}. Note that $[\mW,\mU]$ is a \emph{short} matrix, \emph{i.e.} $[\mW,\mU]\in\mathbb{R}^{m\times(m+d)}$, and the size of $\varepsilon$-net for $[\vz_i^\top,\vx_i^\top]$  is too large for us to derive a mild over-parameterization condition. In order to overcome the technical challenge, we propose a novel probabilistic framework in Section~\ref{sec:initial} by introducing the ``new fresh randomness''~\citep{allen2019convergence-rnn} and perform a fine-grained non-asymptotic analysis on random DEQs.

%% file: result.tex
\subsection{Global Convergence under an Initial Condition}\label{sec:global}
Let $\delta$ be any positive constant such that  $\|\mW(0)\|_2 +\delta <1$.
We   define the following quantities:
    \[\bar\rho_w = \|\mW(0)\|_2 +\delta , \, \bar\rho_u = \|\mU(0)\|_2 + \delta,  \, \bar\rho_a = \|\va(0)\|_2 + \delta.\]
We first present the global convergence of GD by supposing the following condition on the least eigenvalue $\lambda_0$ of the initial Gram matrix $\mG(0)$. 
\begin{cond} At initialization, it holds that
\begin{align}
  &\lambda_0\geq  \frac{4}{\delta} 
      \max\left(c_w,c_u,c_a\right)\left\|\mX\right\|_F\|\hat \vy(0)-\vy\|_2\label{eq:cond21},\\
&\lambda_0^{\frac{3}{2}}\geq 4(2+\sqrt2) \bar\rho_a^{-1}
     \left(
      c_w^2+c_u^2\right)\left\|\mX\right\|_F^2\|\hat \vy(0)-\vy\|_2\label{eq:cond22},\\
  &\lambda_0\geq 4\left(c_w^2+c_u^2\right)\left\|\mX\right\|_F^2,\label{eq:cond23}
\end{align}
where
    $c_w =\frac{\bar\rho_u\bar\rho_a}{(1-\bar\rho_w)^2}$, $c_u =\frac{\bar\rho_a}{1-\bar\rho_w}$, $ c_a =\frac{\bar\rho_u}{1-\bar\rho_w}$.
\label{cond2}
\end{cond}
The convergence result under Condition~\ref{cond2} is presented as follows.
\begin{thm}\label{thm:convergence}
Consider a DEQ defined in~\eqref{eq:deq}. Suppose that Condition~\ref{cond2} holds at initialization.
If the learning rate satisfies
\begin{equation}
\eta <\min\left(\frac{2}{\lambda_0}, \frac{2(c_w^2+c_u^2)}{(c_w^2+c_u^2+c_a^2)^2\left\|\mX\right\|_F^2}\right),
\end{equation}
for every $\tau\geq0$,  the following holds  
\begin{enumerate}
	\item[(\romannumeral1)] $\|\mW(\tau)\|_2<1$, \emph{i.e.} the equilibrium points always exist,
	
	\item[(\romannumeral2)] $\lambda_\tau>\frac{1}{2}\lambda_0$, and thus the PL condition holds as  $\left\|\nabla_{\boldsymbol{\theta}} \Phi(\tau)\right\|_2^2\geq\lambda_0\Phi(\tau)$,
	
	\item[(\romannumeral3)]    the loss converges to a global minimum as \begin{equation*}
	\Phi(\tau) \leq \left(1- \eta \frac{\lambda_0}{2}\right)^\tau \Phi(0).
	\end{equation*}
\end{enumerate}

\end{thm}
Theorem~\ref{thm:convergence} shows that the unique equilibrium point always exists during the training, and  GD converges to a global optimum under Condition~\ref{cond2}. The proof of this part is inspired by the framework proposed in~\citet[Theorem 2.2]{2021On}. The complete proof is deferred to the supplementary material. Next, we discuss how these initial conditions can be fulfilled via over-parameterization. 
\subsection{Initial Condition}
In this section, we aim to show that,   Condition~\ref{cond2} holds  under Assumptions~\ref{assum:initial} and~\ref{assum:data} via over-parameterization.
To do so, it suffices to derive a lower bound on $\lambda_0$, and  upper bounds on $\max\left(c_w,c_u,c_a\right)$, $\bar\rho_a^{-1}$ and  the initial loss $\Phi(0)$, and put these bounds into \eqref{eq:cond21}-\eqref{eq:cond23} to obtain a specific condition on width $m$.

Firstly,  we present
the lower bound of $\lambda_0$ as follows.
\begin{thm}\label{thm:mcondtion}
If $m = \Omega\left(\frac{n^2}{\lambda_*^2}\left( \log\frac{n}{\lambda_* t}\right)\right)$, with probability at least $1-t$, it holds that
\begin{equation*}
    \lambda_0\geq \frac{m}{2}\lambda_*.
\end{equation*} 
\end{thm}
The proof of Theorem~\ref{thm:mcondtion} is based on a concentration inequality between the empirical Gram matrix $\mG$ and the   population Gram matrix $\mK$. Since that the weight matrices are   implicitly reused in DEQs for infinite times (see ~\eqref{eq:deql}),  standard concentration tools are not applicable. In order to  overcome the challenge arising from the weight-sharing and the infinite depth.
We present a novel probabilistic framework.  It is one of our core contributions. Detailed analysis is presented in Section~\ref{sec:initial}.

Secondly,  by standard bounds on the operator norm of Gaussian matrices~\citep{vershynin2018high}, we have \emph{w.p.} $\geq 1- \exp{-\Omega(m)}$ that, $\|\mW(0)\|_2=\order{1}$, $\|\mU(0)\|_2=\order{\sqrt{m/d}}$ under Assumption~\ref{assum:initial}. 
Thus, \emph{w.p.} $\geq 1- \exp{-\Omega(m)}$, it holds that
\begin{equation*}
     \bar\rho_w = \order{1},\quad \bar\rho_u = \order{\sqrt{\frac{m}{d}}}
\end{equation*}
    which implies that 
\begin{equation*}
    \max(c_w,c_u,c_a) = \order{\sqrt{\frac{m}{d}}}.
\end{equation*}
By standard bounds on the norm of Gaussian vector~\citep{vershynin2018high}, we have \emph{w.p.} $\geq 1- \exp{-\Omega(m)}$ that, $\bar\rho_a^{-1}= \order{1}$.

Thirdly, by using the property of the contractive mapping~\eqref{eq:deql} and using the standard concentration argument, it is easy to show that  \emph{w.p.} $\geq 1-t$, it holds that
\[\Phi(0) = \order{n}.\]
Putting all these bounds into~\eqref{eq:cond21}-\eqref{eq:cond23}, one can show that \emph{w.p.} $\geq 1-t$, Condition~\ref{cond2}
is satisfied for $m = \Omega\left(\frac{n^3}{\lambda_*^2}\left( \log\frac{n}{\lambda_* t}\right)\right)$.

%% file: initialization.tex
For notational simplicity, we denote  $\mW(0)$ and $\mU(0)$ by  $\mW$ and $\mU$ in this section.  The Gram matrices of the equilibrium point $\mZ^{(l)}$ of the $l$-th layer  are defined as $\mG^{(l)} \triangleq \left(\mZ^{(l)}\right)^\top\mZ^{(l)}$, for $l\geq1$. Without loss of generality, we assume that $\mZ^{(0)} = 0$. 

Our main idea is to establish the concentration inequality between the empirical Gram matrix $\mG$ and the population Gram matrix $\mK$.
A simple case is to initialize DEQs as  single-layer explicit models. Specifically, take $\sigma_w^2=0$ and  $\mZ = \phi(\mU\mX)$. Using the standard Bernstein inequality, one can  show that \emph{w.p.}$ \geq 1-t$,  $\lambda_0\geq m\tilde{\lambda}/2$, as long as $m = \Omega\left(\tilde{\lambda}^{-2}n^2\log(n/t)\right)$ where $\tilde{\lambda} = \lambda_{\min}\left(\mK^{(1)}\right)$. This case has been  studied in many previous works ~\citep{nguyen2020global,oymak2020toward}.

However, when $\sigma_w^2>0$, previous random analyses on explicit  networks cannot be directly applied to  DEQs. This is due to the fact that these analyses rely on the independence of initial random weights and features, which is no longer the case in DEQs.  In order to overcome these technical difficulties, we propose a novel probabilistic framework based on the following observations:
 $\mK_{ij}^{(l)}$, $\frac{1}{m}\mG_{ij}^{(l)}$, and $\mG_{ij}^{(l)}$ converge to $\mK_{ij}$, $\mK_{ij}^{(l)}$, and $\mG_{ij}$
at an exponential rate, respectively. Moreover, $\mK$ is strictly positive definite.
These observations imply that  it suffices to  take sufficiently large $l$  and $m$ to ensure that $\|\frac{1}{m}\mG-\mK\|_F$ is smaller than $\lambda_*$. Thus, one can  lower bound  $\lambda_0$ by invoking Weyl's inequality. Detailed analysis is given as follows.

\subsection{Bound between infinite-depth and finite-depth weight-untied models}\label{sec:inital1}
In the case of $\phi=\text{ReLU}$, given any positive definite matrix $\mA = \left[\begin{array}{cc}
    1 & x \\
    x & 1 
\end{array}\right]$ with $|x|\leq1$ , and two random variables $(\ru,\rv)^\top\sim\gN(0,\mA)$, as shown in~\citep{NIPS2016_abea47ba}, it holds that 
 \begin{equation}
 \E\left[\phi(\ru)\phi(\rv)\right] = \frac{1}{2}Q(x),  
 \label{eq:Erelu}
 \end{equation}
 where $Q(x)\triangleq\frac{\sqrt{1-x^2}+\left(\pi-\arccos x\right)x}{\pi}$.
  
Combining ~\eqref{eq:Erelu} with the homogeneity of ReLU, we can have more precise expressions of $\mK_{ij}^{(l)}$ as follows.
 \begin{lem}  Let $\cos\theta_{ij}^{(l)} = \frac{\sigma_w^2\mK_{ij}^{(l-1)}+d^{-1}\vx_i^\top\vx_j}{\sigma_w^2\mK_{ij}^{(l-1)}+1}$, and $\rho^{(l)}=\mK_{ii}^{(l)}$.
For  $ l \geq 1$ and $(i,j)\in[n] \times [n]$, $\mK_{ij}^{(l)}$ defined in  Definition~\ref{def:defs} can be written as   
\begin{equation}
    \mK_{ij}^{(l)} = \rho^{(l)}Q\left(\cos\theta_{ij}^{(l)}\right),
    \label{eq:Kdef}
\end{equation}
with $\rho^{(l)}=\frac{1-\sigma_w^{2l}}{1-\sigma_w^2}$ and 
\begin{equation}
    \cos\theta_{ij}^{(l)} =  \left(1-\frac{1}{\rho^{(l-1)}}\right)Q\left(\cos\theta_{ij}^{(l-1)}\right)+ \frac{1}{\rho^{(l-1)}}\frac{\vx_i^\top\vx_j}{d}.
    \label{eq:cosl}
\end{equation}
\label{lem:reluk}
\end{lem}
\begin{thm}\label{thm:K-KL} Under Assumptions~\ref{assum:initial} and~\ref{assum:data},
it holds that
\begin{enumerate}
    \item[(\romannumeral1)] $\left\|\mK- \mK^{(l)}\right\|_F =\order{n l\sigma_w^{2l} }$, which implies that, for $l\rightarrow \infty$, $\mK^{(l)}$ converges to $\mK$ with each entry
\begin{equation}
    \mK_{ij} = \frac{1}{1-\sigma_w^2} Q(\cos \theta_{ij}),
    \label{eq:Kinf}
\end{equation}
where $\cos\theta_{ij} = \sigma_w^2Q(\cos\theta_{ij}) + (1-\sigma_w^2)\frac{\vx_i\vx_j}{d}$.
    \item[(\romannumeral2)]$\mK$ is strictly positive definite, \emph{i.e.} $\lambda_* > 0$.
\end{enumerate}
\end{thm}
\begin{proof}[Sketch of Proof]
(\romannumeral1) By \eqref{eq:defkraw} and \eqref{eq:Kdef}, one can show that $\left|\mK_{ij}^{(l+1)}-\mK_{ij}^{(l)}\right|\leq \sigma_w^2\left|\mK_{ij}^{(l)}-\mK_{ij}^{(l-1)}\right| + 2\sigma_w^{2l}$, which implies that $\left|\mK_{ij}-\mK_{ij}^{(l)}\right| = \order{l\sigma_w^{2l}}$. 

Note that $\sigma_w<1$. Thus $\mK^{(l)}$ converges to a deterministic Gram matrix $\mK$, and one can easily obtain \eqref{eq:Kinf} from \eqref{eq:cosl} by letting $l\rightarrow\infty$. 

(\romannumeral2) The proof of this part is similar with~\citep{oymak2020toward,nguyen2020global}. By performing the Hermite analysis on \eqref{eq:Kinf}, one can  show that  $\mK$ is strictly positive definite if $ |\cos\theta_{ij}|<1$ for all $i \neq j$.
Using the fact that  $|Q(x)|\leq 1$, \eqref{eq:Kinf} implies that $ |\cos\theta_{ij}|<1$ under Assumption~\ref{assum:data}. Please see the complete proof in the supplementary material. 
\end{proof}
\subsection{Bound between infinite-depth and finite-depth weight-tied models}\label{sec:inital2}
 By leveraging the contractility of the transformation defined in~\eqref{eq:deql} and by invoking the standard Bernstein inequality,
we establish a concentration inequality   between the Gram matrix of the $l$-th layer's output and that of the initial equilibrium point as follows.
\begin{thm}\label{thm:iwt2fwt}
Under Assumptions~\ref{assum:initial} and~\ref{assum:data}, with probability at least $1-n^2\exp{-\Omega\left(m\right)}$, it holds
\begin{equation*}
\frac{1}{m}\left\|\mG-\mG^{(l)}\right\|_F =\order{n\left(2\sqrt{2}\sigma_w\right)^l}.
\end{equation*}
\end{thm}
\begin{proof}[Sketch of Poof] Using the contractility of the implicit layer, it is easy to have  \emph{w.p.} $\geq 1-\exp{-\Omega(m)}$, $\|\vz_i^{(l)}\|_2= \order{\|\vz_i^{(1)}\|_2}$, and $\|\vz_i-\vz_i^{(l)}\|_2=\order{(2\sqrt{2}\sigma_w)^l\|\vz_i^{(1)}\|_2}$.
Moreover,
using  Bernstein inequality, we have \emph{w.p.} $\geq 1-\exp{-\Omega\left(mt^2\right)}$, $\bigl|\frac{1}{m}(\vz_i^{(1)})^\top \vz_i^{(1)}-1\bigr|\leq t$.
Let $t$ be an absolute positive constant. Theorem~\ref{thm:iwt2fwt} can be proved by applying the simple union bound. Please see the complete proof in the supplementary material.
\end{proof}
\subsection{Bound between weight-tied and weight-untied models with finite-depth}\label{sec:inital3}
Due to the implicit weight-tied structure of DEQs,  standard concentration inequalities built upon the independence cannot be directly applied.
    In order to  overcome  technical difficulties, we build  necessary probabilistic tools for DEQs.
 We  first introduce Lemma~\ref{lem:rec} which provides the ``fresh new randomness'' for our analysis.
The construction method in Lemma~\ref{lem:rec} is inspired by the previous work~\citep{allen2019convergence-rnn} on RNNs.  

\begin{lem} For $l\geq 1$, $\mG_{ij}^{(l+1)}$ can be reconstructed as  $\mG_{ij}^{(l+1)} = \phi(\mM\vh)^\top\phi(\mM\vh')$ such that 

(\romannumeral1) $\vh^\top\vh' =\frac{\sigma_w^2}{m}\mG_{ij}^{(l)} + \frac{1}{d}\vx_i^\top\vx_j$, 

(\romannumeral2)  $\mM\in \mathbb{R}^{m\times (2l+d+2)}$ is a \emph{rectangle matrix}, and the entries of $\mM$ are \emph{i.i.d.} from $\gN(0,2)$ conditioning on previous layers.
\label{lem:rec}
\end{lem}
\begin{proof} (\romannumeral1) 
Let  $ \mV_{l} \in\mathbb{R}^{m \times 2l } $ denote a column orthonormal matrix using Gram-Schmidt as \begin{equation*}
	\mV_{l} = \operatorname{GS}\left(\vz_{i}^{(1)},\cdots,\vz_{i}^{(l)},\vz_{j}^{(1)},\cdots,\vz_{j}^{(l)}\right).
\end{equation*}
For each $i,j$, we define
$\vp \triangleq (\mI-\mV_{l-1}\mV_{l-1}^\top)\vz_i^{(l)}$ and $\vq \triangleq (\mI-\mV_{l-1}\mV_{l-1}^\top)\vz_j^{(l)}$.
We split $\vq$ into two parts $    \vq = \vq^{\parallel} + \vq^{\perp}$, where $\vq^{\parallel}$ is parallel to $\vp$ and $\vq^{\perp}$ is orthogonal to $\vp$ as
\[\vq^{\parallel}=\frac{\vp^\top\vq}{\|\vp\|_2^2}\vp, \quad \vq^{\perp} = \left(\mI-\frac{\vp\vp^\top}{\|\vp\|_2^2}\right)\vq.\]
First, we construct  $\mM$ as  $\mM=[\mM_1,\mM_2,\mM_3,\mM_4]$ with 
\begin{align*}
\begin{array}{cc}
\mM_1 = \sigma_w^{-1}\sqrt{m}{\mW\mV_{l-1}}, &\mM_2=\sigma_w^{-1}\sqrt{m}\mW\frac{\vp}{\|\vp\|_2}, \\
 \mM_3 = \sigma_w^{-1}\sqrt{m}\mW\frac{\vq^\perp}{\|\vq^\perp\|_2}, &\mM_4= \sqrt{d}\mU. 
\end{array}
\end{align*}
Then,  we construct $\vh =[\vh_1^\top,\vh_2^\top,\vh_3^\top,\vh_4^\top]^\top$ 
 with 
\begin{align*}
\begin{array}{cc}\vh_1 = \frac{\sigma_w}{\sqrt{m}}\mV_{l-1}^\top\vz_i^{(l)},& \vh_2 = \frac{\sigma_w}{\sqrt{m}}\|\vp\|_2,\\   \vh_3 =0,& \vh_4 = \frac{1}{\sqrt{d}}\vx_i,\end{array}
\end{align*}
and $\vh'=[\vh_1^{'\top},\vh_2^{'\top},\vh_3^{'\top},\vh_4^{'\top}]^\top$ with
\begin{align*}
\begin{array}{cc}\vh'_1 = \frac{\sigma_w}{\sqrt{m}}\mV_{l-1}^\top\vz_j^{(l)}, & \vh'_2 = \frac{\sigma_w\vp^\top\vq}{\sqrt{m}\|\vp\|_2}, \\ \vh'_3 =\frac{\sigma_w}{\sqrt{m}}\|\vq^{\perp}\|_2,&\vh'_4 = \frac{1}{\sqrt{d}}\vx_j.
\end{array}
\end{align*}
It is easy to check that $\mG_{ij}^{(l+1)} = \phi(\mM\vh)^\top\phi(\mM\vh')$ with $\vh^\top\vh' =\frac{\sigma_w^2}{m}\mG_{ij}^{(l)} + \frac{1}{d}\vx_i^\top\vx_j.$

(\romannumeral2) Let  $\mV_{l-1} \triangleq \left[\vv_1,\cdots,\vv_{2(l-1)}\right]$. Note that $\vv_j$ only depends on the randomness of $\mU$ and  $\mW\left[\vv_1,\cdots,\vv_{j-1}\right]$. This means that, conditioning on $\mU$ and $\mW\left[\vv_1,\cdots,\vv_{j-1}\right]$,  $\mW\vv_j$ is still an independent Gaussian vector. Similarly, we have $\mW\vp$ and $\mW\vq$ are also independent Gaussian vectors. Consequently, we prove that the entries of $\mM$ are \emph{i.i.d.} from $\gN(0, 2)$.
\end{proof}
We stress that although $\mM$ constructed in Lemma~\ref{lem:rec} has \emph{i.i.d.} entries, it still depends on $\vh$ and $\vh'$. Thus, the standard Bernstein inequality cannot be directly applied. To address this issue, we  perform the standard $\varepsilon$-argument~\citep{vershynin2018high} and leverage the   ``fresh new randomness''~\citep{allen2019convergence-rnn} provided in Lemma~\ref{lem:rec}.  The concentration inequality between $\frac{1}{m}\mG^{(l)}$ and $\mK^{(l)}$  is established as follows. 
\begin{thm}\label{thm:G-K} Under Assumptions~\ref{assum:initial} and~\ref{assum:data}, 
with probability at least $1-n^2\exp{-\Omega\left(m8^l\sigma_w^{2l}\right)+\order{l^2}}$,  it holds that
\begin{equation*}
    \left\|\frac{1}{m}\mG^{(l)}-\mK^{(l)}\right\|_F\leq n\left(2\sqrt{2}\sigma_w\right)^l.
\end{equation*}
\end{thm}
\begin{proof}[Sketch of  Proof]
We give the main ideas of the proof.

 It holds that $\bigl|\frac{1}{m}\mG_{ij}^{(l)}-\mK_{ij}^{(l)}\bigr|\leq\bigl|\frac{1}{m}\mG_{ij}^{(l)}-\E\bigl[\frac{1}{m}\mG_{ij}^{(l)}\bigr]\bigr|+\bigl|\E\bigl[\frac{1}{m}\mG_{ij}^{(l)}\bigr]-\mK_{ij}^{(l)}\bigr|$ by the triangle inequality. Following Lemma~\ref{lem:rec},  we reconstruct $\mG_{ij}^{(l)} = \phi(\mM\vh)^\top\phi(\mM\vh')$.
\begin{enumerate}
\item[(\romannumeral1)] For \emph{fixed} $\vh$ and $\vh'$, using the    Bernstein inequality, and on can show that \emph{w.p.} $\geq 1-\exp{-\Omega(m\varepsilon^2)}$, it holds that 
\[\left|\frac{1}{m}\mG_{ij}^{(l)}-\E[\frac{1}{m}\mG_{ij}^{(l)}]\right|\leq \varepsilon.\]
\item[(\romannumeral2)] For \emph{all} $\vh$ and $\vh'$, we apply the standard $\varepsilon$-net argument.  Note that the size of $\varepsilon$-net for $\vh,\vh'$ is at most $\exp{\order{l\log \frac{1}{\varepsilon}}}$. Thus, one can derive that  \emph{w.p.} $\geq 1-\exp{-\Omega(m\varepsilon^2)+\order{l\log \frac{1}{\varepsilon}}}$, it holds that  \[\left|\frac{1}{m}\mG_{ij}^{(l)}-\E\bigl[\frac{1}{m}\mG_{ij}^{(l)}\bigr]\right|\leq \varepsilon.\]
\item[(\romannumeral3)] Substitute the choice of  $\vh$ and $\vh'$ such that $\vh^\top\vh' =\frac{\sigma_w^2}{m}\mG_{ij}^{(l-1)} + \frac{1}{d}\vx_i^\top\vx_j$. Using the fact that $\E\bigl[\frac{1}{m}\mG_{ij}^{(l)}\bigr]$ is determined  by $\vh^\top\vh'$, one can further derive that \emph{w.p.} $\geq  1-l\exp{-\Omega \left(m\varepsilon^2\right)+\order{l\log \frac{1}{\varepsilon}}}$,  \[\left|\E\bigl[\frac{1}{m}\mG_{ij}^{(l)}\bigl]-\mK_{ij}^{(l)}\right| \leq \sigma_w^2\left|\frac{1}{m}\mG_{ij}^{(l-1)}-\mK_{ij}^{(l-1)}\right| + \varepsilon,\] which implies that \[\left|\frac{1}{m}\mG_{ij}^{(l)}-\mK_{ij}^{(l)}\right|\leq \sigma_w^2\left|\frac{1}{m}\mG_{ij}^{(l-1)}-\mK_{ij}^{(l-1)}\right| + 2\varepsilon.\]
Note that  $\sigma_w^2< 1/8$, and thus  a simple induction argument works here.
\end{enumerate}
Lastly, let $\varepsilon=(2\sqrt{2}\sigma_w)^l$,  and  Theorem~\ref{thm:G-K} can be proved by  using the simple union bound.  Please see the complete proof in the supplementary material.
\end{proof}
\begin{figure*}
\centering
\subfigure[Synthetic data ]{\begin{minipage}[c]{0.33\textwidth}
\centering
\includegraphics[width=1\textwidth]{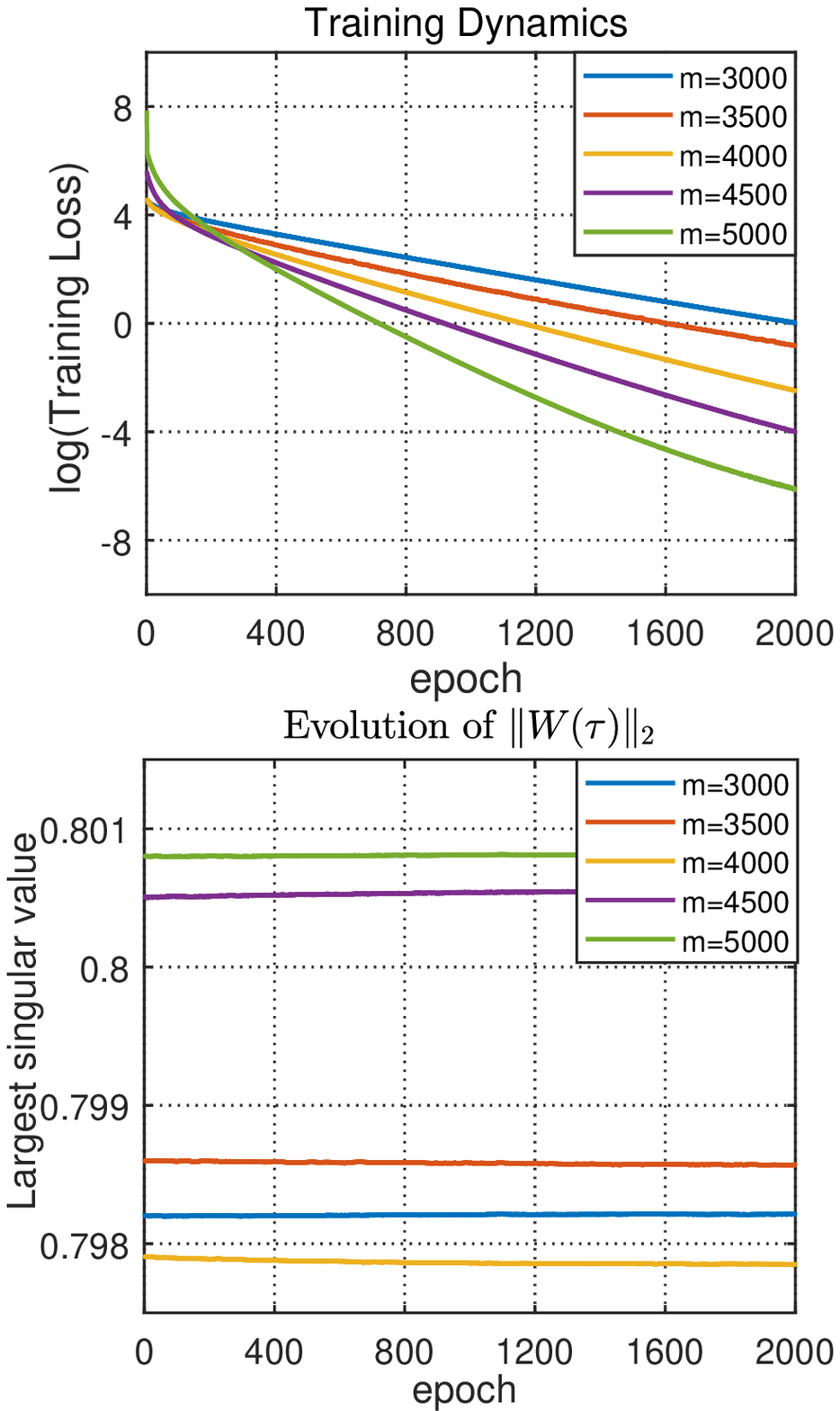}
\end{minipage}}
% \hspace{0.01\textwidth}
\subfigure[MNIST]{\begin{minipage}[c]{0.33\textwidth}
\centering
\includegraphics[width=1\textwidth]{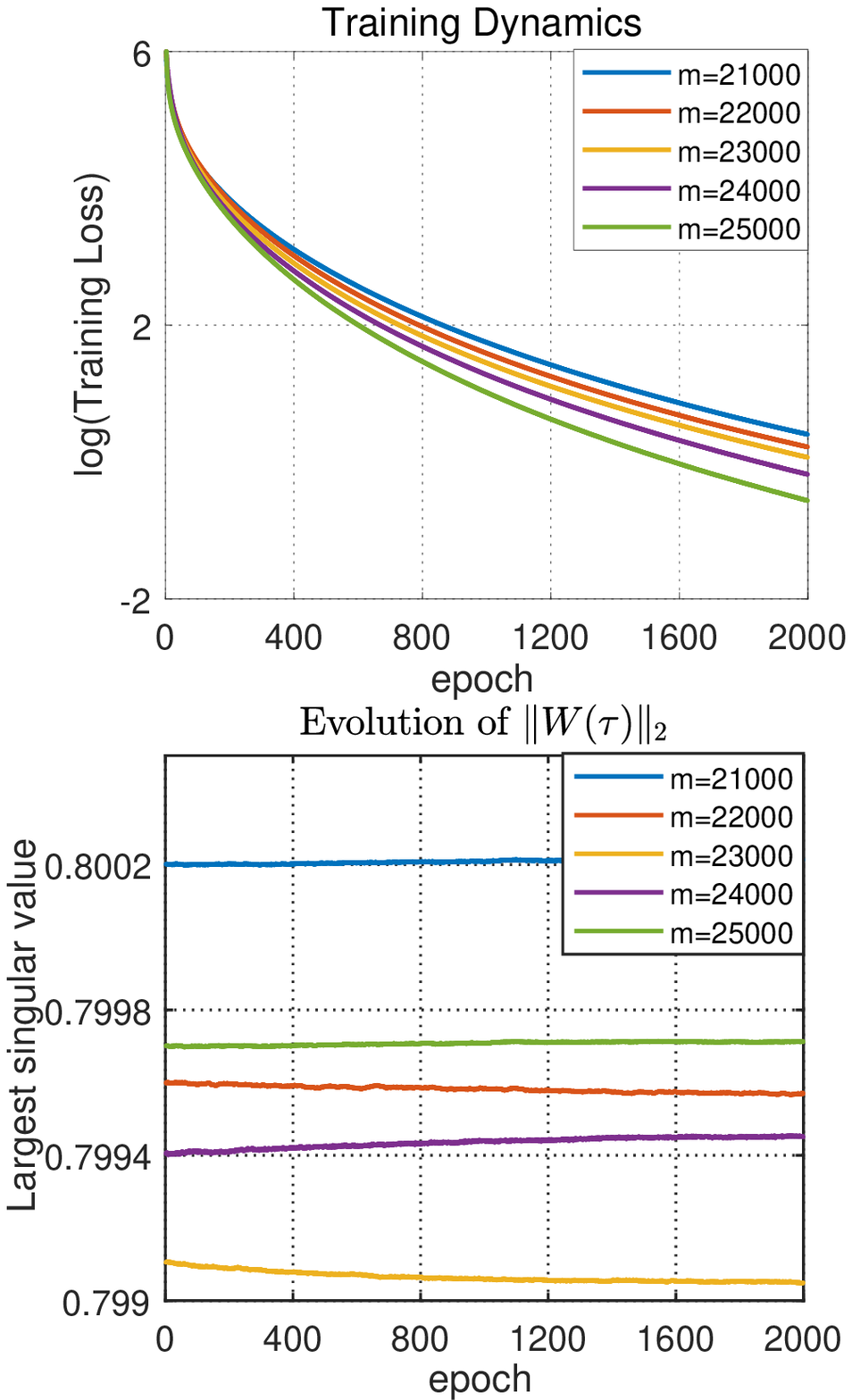}
\end{minipage}}
% \hspace{0.01\textwidth}
\subfigure[CIFAR10]{\begin{minipage}[c]{0.33\textwidth}
\centering
\includegraphics[width=1\textwidth]{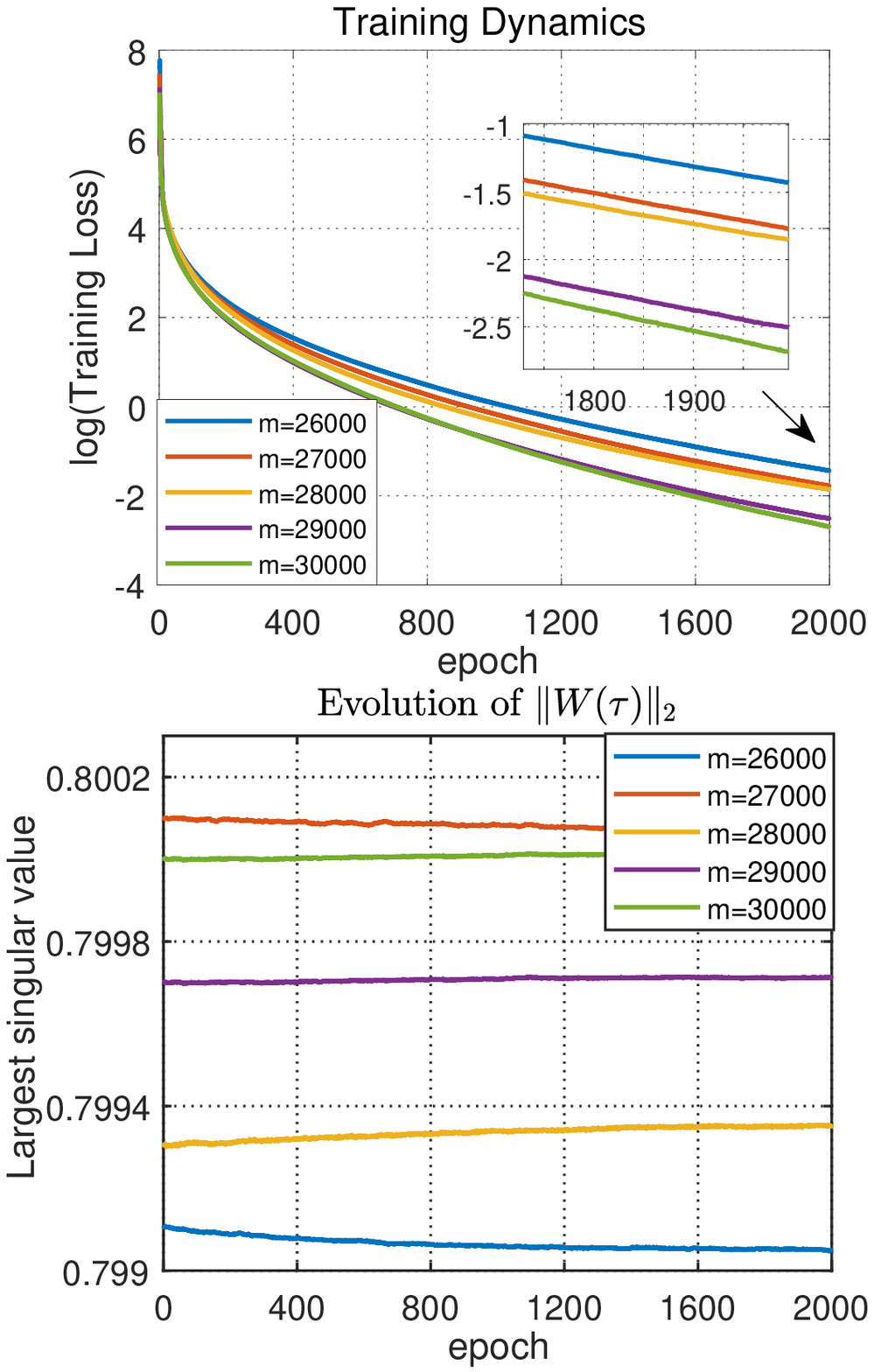}
\end{minipage}}
\caption{Results of different widths   on (a) Synthetic data; (b) MNIST; (c) CIFAR10.}
\label{experimentFig1}
\end{figure*}
\subsection{Proof of Theorem~\ref{thm:mcondtion}}
Firstly, combining  Theorems~\ref{thm:K-KL},~\ref{thm:iwt2fwt} and~\ref{thm:G-K}, one can show that \emph{w.p.}  $\geq 1-n^2\exp{-\Omega\left(m8^l\sigma_w^{2l}\right)+\order{l^2}}$, it holds
\begin{align*}
     &\bigl\|\frac{1}{m}\mG-\mK\bigr\|_F\\
	\leq& \frac{1}{m}\bigl\|\mG-\mG^{(l)}\bigr\|_F + \bigl\|\frac{1}{m}\mG^{(l)}-\mK^{(l)}\bigr\|_F + \bigl\|\mK-\mK^{(l)}\bigr\|_F\\
	=&\order{n\left(2\sqrt{2}\sigma_w\right)^l} + \order{n\left(2\sqrt{2}\sigma_w\right)^l}+ \order{nl\sigma_w^{2l}}\\
	=&\order{n\left(2\sqrt{2}\sigma_w\right)^l},
\end{align*}
where the last equality comes from the fact that $l\sigma_w^{l}\leq(2\sqrt{2})^l $, for $l\geq 1$.

Next, we fix  $l$ to omit the explicit dependence in $l$. Specifically, we take $l=\Theta\left( \log (\lambda_*^{-1}n)/\log({\sqrt{2}}/{4\sigma_w} )\right)$. The lower bound of 
   $l$  is large enough to ensure that $\left\|\frac{1}{m}\mG-\mK\right\|_F\leq\frac{\lambda_*}{2}$. Therefore, by Weyl's inequality, it holds that $\lambda_{0}>\frac{m}{2}\lambda_*$. Meanwhile, the upper bound of $l$  guarantees that the probability does not decrease exponentially. Thus, one can use a mild over-parameterization condition on $m$ to ensure the high probability.  

 Consequently, one can show that, it holds \emph{w.p.} $\geq 1-t$, $\lambda_0\geq \frac{m}{2}\lambda_*$,
 as long as  $m = \Omega\left(\frac{n^2}{\lambda_*^2}\left( \log\frac{n}{\lambda_* t}\right)\right)$.

% \begin{figure*}
% \centering
% \subfigure[Synthetic data ]{\begin{minipage}[c]{0.33\textwidth}
% \centering
% \includegraphics[width=1\textwidth]{fig/random_new.eps}
% \end{minipage}}
% % \hspace{0.01\textwidth}
% \subfigure[MNIST]{\begin{minipage}[c]{0.33\textwidth}
% \centering
% \includegraphics[width=1\textwidth]{fig/mnist_new.eps}
% \end{minipage}}
% % \hspace{0.01\textwidth}
% \subfigure[CIFAR10]{\begin{minipage}[c]{0.33\textwidth}
% \centering
% \includegraphics[width=1\textwidth]{fig/cifar_new.eps}
% \end{minipage}}
% \caption{Results of different widths   on (a) Synthetic data; (b) MNIST; (c) CIFAR10.}
% \label{experimentFig1}
% \end{figure*}

\paragraph{Discussion on comparisons with~\citet{2021A} and~\citet{gao2022}.}
In~\citet{2021A},  the prediction of an implicit network is formulated as a weighted summation of equilibrium points  and explicit features \emph{i.e.} $\hat{\vy}=\va^\top\mZ+\vb^\top\phi(\mU\mX)$.  Therefore,  it is hard to measure the contribution of the equilibrium point to the capacity of implicit models in their case.
In contrast, we only use  equilibrium points for  predictions,  and our  result illustrates that arbitrary  training label can be fitted only using  equilibrium points. In a concurrent work~\citep{gao2022} on ReLU implicit networks, they consider ``a subset of initialization''  which requires entries of $\mW$ to be non-negative. In their case, one can show that $\mZ = (\mI-\mW)^{-1}\phi(\mU\mX)$ at initialization. This implies that the initial equilibrium point is essentially a linear transformation of the explicit feature $\phi(\mU\mX)$. In  contrast, we consider general Gaussian  initialization which is commonly used in practice.
 The  results of~\citet{2021A} and~\citet{gao2022} can be built upon the lower bound of the least singular value of $\phi(\mU\mX)$, which is given by previous works on explicit models.
However,  previous analyzes do not applied in our case.  In order to address the technical problem, we  propose a novel probabilistic framework. Therefore,  our studies have distinct differences.

%% file: exp.tex
% \begin{figure*}
% \centering
% \subfigure[Synthetic data ]{\begin{minipage}[c]{0.33\textwidth}
% \centering
% \includegraphics[width=1\textwidth]{fig/random_new.eps}
% \end{minipage}}
% % \hspace{0.01\textwidth}
% \subfigure[MNIST]{\begin{minipage}[c]{0.33\textwidth}
% \centering
% \includegraphics[width=1\textwidth]{fig/mnist_new.eps}
% \end{minipage}}
% % \hspace{0.01\textwidth}
% \subfigure[CIFAR10]{\begin{minipage}[c]{0.33\textwidth}
% \centering
% \includegraphics[width=1\textwidth]{fig/cifar_new.eps}
% \end{minipage}}
% \caption{Results of different widths   on (a) Synthetic data; (b) MNIST; (c) CIFAR10.}
% \label{experimentFig1}
% \end{figure*}
In this section, we implement several numerical experiments to verify our main theoretical conclusions. We first evaluate our
 method on various datasets including  synthetic data,  MNIST, and CIFAR10. For constructing synthetic data, we uniformly generate $n = 1000$ data points from a
 $d = 1000$ dimensional  sphere with radius $\sqrt{d}$, and labels are generated from a one-dimensional standard Gaussian
distribution. For  each dataset of MNIST and CIFAR10, we  randomly sample $500$ images from each of class $0$ and class $1$ to generate the training dataset with $n = 1000$ samples. We  use Gaussian initialization as suggested in Assumption~\ref{assum:initial} and $\sigma_w^2$ is set as $0.08$.
We  normalize each data point as suggested in Assumption~\ref{assum:data}.

In the first experiment, we test how the width  affects the  convergence rate. 
% For synthetic data, we set  $m\in\{ 3000, 3500, 4000, 4500,5000\}$. For MNIST, we set $m \in \{ 21000, 22000,23000, 24000, 25000\}$. For CIFAR10, we set  $m \in \{ 26000, 27000,28000, 29000, 30000\}$. 
As shown in Figure~\ref{experimentFig1},  the convergence speed becomes faster as $m$ increases, and the final training loss becomes smaller. 
 We believe that the reason is as $m$ increases, Gram matrices become more stable.
The second experiment verifies that $\|\mW(\tau)\|_2$ is smaller than $1$ throughout the training, which implies that the unique equilibrium point always exists.

%% file: conclusion.tex
In this paper,  we analyze the
gradient dynamics of DEQs with the quadratic loss function.
Under a specific initial condition,   we prove that the GD converges to a global optimum at a linear rate. By performing a fine-grained analysis on  Gaussian initialized  DEQs, we further show that the initial conditions are satisfied via mild over-parameterization. Specifically, we present a new probabilistic framework to address the  challenge arising from  the  weight-sharing and the infinite depth.
 To the best of our knowledge, it is the first time to analyze the equilibrium point of a random DEQ.
Moreover, we show that the unique equilibrium points  always exist during the training process. Our analysis is specific to ReLU DEQs. For future research,
it would be  interesting to generalize the result to other nonlinear activations. Moreover,
it would be  interesting to explore the generalization performance of over-parameterized DEQs based on our  analysis at initialization.   
\subsubsection*{Acknowledgements}
Z. Lin was supported by the major key project of PCL, China (No. PCL2021A12), the NSF China (No. 62276004), and Project 2020BD006 supported by PKU-Baidu Fund, China.

%% file: appendix.tex
\onecolumn
% \aistatstitle{Instructions for Paper Submissions to AISTATS 2022: \\
% Supplementary Materials}
\section{USEFUL TECHNICAL LEMMAS}
\begin{lem}[Weyl's inequality]
	Let $\mA,\mB\in\mathbb{R}^{m\times n}$ with $\sigma_1(\mA)\geq\cdots\geq\sigma_r(\mA)$ and $\sigma_1(\mB)\geq\cdots\geq\sigma_r(\mB)$, where $r=\min(m,n)$. Then,
	\begin{align*}
		\max_{i\in[r]}\left|\sigma_i(\mA)-\sigma_i(\mA)\right|\leq \left\|\mA-\mB\right\|_2.
	\end{align*}
\end{lem}
\begin{lem}\label{lem:hadma}  Let us define $\odot$ as the Hadamard product.
	Given two positive semi-definite (PSD) matrices $\mA$ and $\mB$, it holds that 
	\begin{align*}
	\lambda_{\min }(\boldsymbol{A} \odot \boldsymbol{B}) \geq\left(\min _{i} \boldsymbol{B}_{i i}\right) \lambda_{\min }(\boldsymbol{A}).
	\end{align*}
\end{lem}
\begin{lem}[\citet{NIPS2016_abea47ba}]  Let $h_{r}(x)=\frac{1}{\sqrt{r !}}(-1)^{r} e^{x^2 / 2} \frac{d^{r}}{d x^{r}} e^{-x^{2} / 2}$ be normalized probabilist’s hermite polynomials. Let $\phi(\cdot)$ denote ReLU , we define $\mu_{r}(\phi)=\int_{-\infty}^{\infty} \phi(x) h_{r}(x) \frac{e^{-x^{2} / 2}}{\sqrt{ \pi}} dx$. It holds that 
	\begin{align*}
	Q(x) = \sum_{r=0}^{\infty} \mu_r^2(\phi)x^{r} =\frac{\sqrt{1-x^2}+\left(\pi-\arccos x\right)x}{\pi}.
	\end{align*}
	Moreover, it holds that $\sup \{r:\mu_r^2(\phi)>0\}=\infty$.
	%where 
	%\begin{equation}
	%\mu_r^2(\phi)=\left\{\begin{array}{ll}
	%\frac{((r-3) ! !)^{2}}{\pi r !} & \text { if } r \text { is even } \\
	%\frac{1}{2} & \text { if } r=1 \\
	%0 & \text { if } r \text { is odd } \geq 3
	%\end{array}\right..
	%\label{eq:mur}
	%\end{equation}
	\label{lem:herrelu}
\end{lem}

\section{PROOF FOR SECTION~\ref{sec:global}}\label{prof:thm1}
We first present several useful inequalities. The proof mainly relies on  basic norm inequalities and  the Lipschitz property of ReLU. 
\begin{lem} For each $s \in [0,\tau]$, suppose that   $  \|\mW(s)\|_2\leq \bar\rho_w$, $ \|\mU(s)\|_2\leq \bar\rho_u$, and $\|\va(s)\|_2\leq  \bar\rho_a $. 
It holds that 
\begin{equation}\label{eq:Zn}
    \|\mZ(s)\|_F\leq c_a\left\|\mX\right\|_F,
\end{equation}
and 
\begin{equation}
\left\{\begin{array}{l}
    \left\|\nabla_\mW \Phi(s)\right\|_F \leq c_w\left\|\mX\right\|_F\left\|\hat{\vy}(\tau) - \vy \right\|_2\\
    \left\|\nabla_\mU\Phi(s)\right\|_F \leq c_u\left\|\mX\right\|_F\left\|\hat{\vy}(\tau) - \vy \right\|_2\\
    \left\|\nabla_\va \Phi(s)\right\|_2 \leq c_a \left\|\mX\right\|_F\left\|\hat{\vy}(\tau) - \vy \right\|_2.
\end{array}\right.
\label{eq:gradbound}
\end{equation}
Furthermore,  for each $k,s\in[0,\tau]$,  it holds that
\begin{equation}
\left\|\mZ(k)-\mZ(s)\right\|_F\leq \bar\rho_a^{-1} \left(c_w\left\|\mW(k)-\mW(s)\right\|_2 + c_u\left\|\mU(k)-\mU(s)\right\|_2\right) \left \|\mX\right\|_F,
\label{eq:deltaz}
\end{equation}
and 
\begin{equation}\label{eq:deltay}
\begin{split}
  &\left\|\hat\vy(k)-\hat\vy(s)\right\|_2\\ 
\leq& \left(c_w\left\|\mW(k)-\mW(s)\right\|_2 + c_u\left\|\mU(k)-\mU(s)\right\|_2+c_a\|\va(k)-\va(s)\|_2\right) \left \|\mX\right\|_F .
\end{split}
\end{equation}
\label{lem:gradbound}
\end{lem}
\begin{proof}
(1) Proof of~\eqref{eq:Zn}: 
Note that $\mZ(s) =\phi(\mW(s)\mZ(s)+\mU(s)\mX)$. Using the fact that $|\phi(x)|\leq|x|$, we have
\begin{align*}
    \left\|\mZ(s)\right\|_F\leq \left( \left\|\mW(s)\right\|_2\left\|\mZ(s)\right\|_F+\left\|\mU(s)\right\|_2\left\|\mX\right\|_F\right)
    \leq\bar\rho_w\left\|\mZ(s)\right\|_F + \bar\rho_u\left\|\mX\right\|_F.
\end{align*}
Note that $\|\mW(s)\|_2\leq \bar{\rho}_w<1$, for each $s\in[0,\tau]$, and thus it holds
\begin{align*}
    \left\|\mZ(s)\right\|_F\leq \frac{\bar\rho_u}{1-\bar\rho_w}\left\|\mX\right\|_F=c_a\left\|\mX\right\|_F.
\end{align*}

(2)  Proof of ~\eqref{eq:gradbound}:
% Recall that 
% \begin{align*}
% \left\{\begin{array}{l}
% \nabla_\mW \Phi(\tau) = \left(\mZ(\tau)\otimes\mI_m\right)\mR(\tau)^\top\left(\hat{\vy}(\tau) - \vy\right)\\
% \nabla_\mU \Phi(\tau) = \left(\mX(\tau)\otimes\mI_m\right)\mR(\tau)^\top\left(\hat{\vy}(\tau) - \vy\right)\\
% \nabla_\va \Phi(\tau) = \mZ(\tau) \left(\hat{\vy}(\tau) - \vy\right)
% \end{array}\right.
% \end{align*}
First, we have 
\begin{equation*}
    \|\mJ(\tau)^{-1}\|_2\leq \frac{1}{1-\bar\rho_w}, 
\end{equation*}
and thus it holds that \[\|\mR(\tau)\|_2\leq \| \|\va(\tau)\|_2\|\mJ(\tau)^{-1}\|_2\|\mD(\tau)\|_2\leq
\frac{\bar\rho_a}{1-\bar\rho_w}.\]Then, we have
\begin{align*}
    \|\nabla_\mW \Phi(\tau)\|_F&= \|\operatorname{vec}(\nabla_\mW \Phi(\tau))\|_2\\
    &= \left\|\left(\mZ(\tau)\otimes\mI_m\right)\mR(\tau)^\top\left(\hat{\vy}(\tau) - \vy\right)\right\|_2\\
    &\leq \left\|\mZ(\tau) \right\|_2 \left\|\mR(\tau) \right\|_2 \left\|\hat{\vy}(\tau) - \vy \right\|_2\\
    &\leq \frac{\bar\rho_u\bar\rho_a}{(1-\bar\rho_w)^2}\left\|\mX\right\|_F\left\|\hat{\vy}(\tau) - \vy \right\|_2,
\end{align*}
\begin{align*}
    \left\|\nabla_\mU \Phi(\tau) \right\|_F&= \left\|\operatorname{vec}\left(\nabla_\mU \Phi(\tau)\right) \right\|_2\\
    &= \left\|\left(\mX(\tau)\otimes\mI_m\right)\mR(\tau)^\top\left(\hat{\vy}(\tau) - \vy\right)\right\|_2\\
    &\leq \frac{\bar\rho_a}{1-\bar\rho_w}\left\|\mX\right\|_F\left\|\hat{\vy}(\tau) - \vy \right\|_2,
\end{align*}
\begin{align*}
    \left\|\nabla_\va \Phi(\tau) \right\|_2&=\left\|\mZ(\hat{\vy}(\tau) - \vy )\right\|_2\leq \frac{\bar\rho_u}{1-\bar\rho_w}\left\|\mX\right\|_F\left\|\hat{\vy}(\tau) - \vy \right\|_2.
\end{align*}
(3) Proof of \eqref{eq:deltaz}: 
\begin{align*}
    &\|\mZ(k)-\mZ(s)\|_F\\ =& \left\|\phi(\mW(k)\mZ(k)+\mU(k)\mX)-\phi(\mW(s)\mZ(s)+\mU(s)\mX)\right\|_F\\
    \leq& \left \|\mW(k)\mZ(k)+\mU(k)\mX-\mW(s)\mZ(s)-\mU(s)\mX\right\|_F\\
    \leq& \left(\left \|\mW(k)\mZ(k)-\mW(k)\mZ(s)\right\|_F +  \left \|\mW(k)\mZ(s)-\mW(s)\mZ(s)\right\|_F +  \left \|\mU(k)\mX-\mU(s)\mX\right\|_F\right)\\
    \leq& \|\mW(k)\|_2\|\mZ(k)-\mZ(s)\|_F + \left( \left \|\mW(k)-\mW(s)\right\|_2  \left \|\mZ(s)\right\|_F + \left \|\mU(k)-\mU(s)\right\|_2  \left \|\mX\right\|_F \right)\\
    \leq& \bar\rho_w\|\mZ(k)-\mZ(s)\|_F + \left(\frac{\bar\rho_u}{1-\bar\rho_w} \left \|\mW(k)-\mW(s)\right\|_2 \left \|\mX\right\|_F + \left \|\mU(k)-\mU(s)\right\|_2  \left \|\mX\right\|_F \right)\\
\end{align*}
Consequently, we have 
\begin{align*}
    \|\mZ(k)-\mZ(s)\|_F\leq \bar\rho_a^{-1} \left(c_w\left\|\mW(k)-\mW(s)\right\|_2 + c_u\left\|\mU(k)-\mU(s)\right\|_2\right) \left \|\mX\right\|_F\\
\end{align*}
(4) Proof of \eqref{eq:deltay}:
\begin{align*}
    &\left\|\hat\vy(k)-\hat\vy(s)\right\|_2 \\
    =& \left\|\va(k)\mZ(k)-\va(s)\mZ(s) \right\|_F\\
    \leq& \left\|\va(k)\mZ(k)-\va(k)\mZ(s) \right\|_F + \left\|\va(k)\mZ(s)-\va(s)\mZ(s) \right\|_F\\
    \leq& \|\va(k)\|_2\|\mZ(k)-\mZ(s)\|_F + \|\va(k)-\va(s)\|_2\|\mZ\|_F\\
    \leq&\left(c_w\left\|\mW(k)-\mW(s)\right\|_2 + c_u\left\|\mU(k)-\mU(s)\right\|_2+c_a\|\va(k)-\va(s)\|_2\right) \left \|\mX\right\|_F,
    \end{align*}
where the last inequality follows from~\eqref{eq:deltaz}.
\end{proof}
\subsection{Proof of  Theorem~\ref{thm:convergence}}
\begin{proof}

We show by induction for every $\tau>0$, 
\begin{equation}
\left\{\begin{array}{l}
    \|\mW(s)\|_2 \leq \bar\rho_w, \,  \|\mU(s)\|_2 \leq \bar\rho_u,\,  \|\va(s)\|_2 \leq \bar\rho_a,\, s\in[0,\tau]\\
    \lambda_s \geq \frac{\lambda_0}{2}, \quad s\in[0,\tau]\\
   \Phi(s+1) \leq \left(1- \eta \frac{\lambda_0}{2}\right)^s\Phi(0), \quad s\in[0,\tau]
\end{array}\right.
\label{eq:induction}
\end{equation}
For $\tau=0$, it is clear  that \eqref{eq:induction} holds.
Assume that ~\eqref{eq:induction} holds up to  $\tau$ iterations.

(1) With the triangle inequality,
\begin{equation}\nonumber
\begin{split}
     \left\|\mW(\tau+1)-\mW(0)\right\|_F 
     \leq&\sum_{s =0 }^{\tau}\left\|\mW(s+1)-\mW(s)\right\|_F \\
     =&\sum_{s =0 }^{\tau}\eta \left\|\nabla_\mW \Phi(s)\right\|_F\\
     \leq& \eta c_w\left\|\mX\right\|_F \sum_{s=0}^{\tau}\left\|\hat\vy (s) - \vy\right\|_2\\
     \leq&  \eta c_w\left\|\mX\right\|_F\sum_{s=0}^{\tau} \left(1- \eta \frac{\lambda_0}{2}\right)^{s/2}\|\hat \vy(0)-\vy\|_2,
\end{split}
\end{equation}
where the second inequality follows from \eqref{eq:gradbound}, and the last
one follows from induction assumption.
Let $u \triangleq \sqrt{1-\eta \lambda_0/2}$. Then  $\left\|\mW(\tau+1)-\mW(0)\right\|_F$ can be bounded with 
\begin{align*}
\frac{2}{\lambda_0} (1-u^2)\frac{1-u^{\tau+1}}{1-u}
     c_w\left\|\mX\right\|_F \|\hat \vy(0)-\vy\|_2 
     \leq \frac{4}{\lambda_0} 
      c_w\left\|\mX\right\|_F\|\hat \vy(0)-\vy\|_2\leq \delta, \quad \text{by \eqref{eq:cond21}}.
\end{align*}
With Weyl’s inequality, it is easy to have $\|\mW(\tau+1)\|_2\leq \bar\rho_w<1$. 

Using the similar technique, one can show that 
\begin{align*}
     &\left\|\mU(\tau+1)-\mU(0)\right\|_F \leq\sum_{s =0 }^{\tau}\left\|\mU(s+1)-\mU(s)\right\|_F \\
     &=\sum_{s =0 }^{\tau}\eta \left\|\nabla_\mU \Phi(s)\right\|_F\leq c_u\left\|\mX\right\|_F \sum_{s=0}^{\tau}\left\|\hat\vy (s) - \vy\right\|_2\\
     &\leq  \eta c_w\left\|\mX\right\|_F\sum_{s=0}^{\tau} \left(1- \eta \frac{\lambda_0}{2}\right)^{s/2}\|\hat \vy(0)-\vy\|_2\\
     &\leq \frac{4}{\lambda_0} 
      c_u\left\|\mX\right\|_F\|\hat \vy(0)-\vy\|_2\leq \delta, \quad \text{by \eqref{eq:cond22}},
\end{align*}
\begin{align*}
     &\left\|\va(\tau+1)-\va(0)\right\|_F \leq\sum_{s =0 }^{\tau}\left\|\va(s+1)-\va(s)\right\|_F \\
     &=\sum_{s =0 }^{\tau}\eta \left\|\nabla_\va \Phi(s)\right\|_F\leq c_a\left\|\mX\right\|_F \sum_{s=0}^{\tau}\left\|\hat\vy (s) - \vy\right\|_2\\
     &\leq  \eta c_w\left\|\mX\right\|_F\sum_{s=0}^{\tau} \left(1- \eta \frac{\lambda_0}{2}\right)^{s/2}\|\hat \vy(0)-\vy\|_2\\
     &\leq \frac{4}{\lambda_0} 
      c_a\left\|\mX\right\|_F\|\hat \vy(0)-\vy\|_2\leq \delta, \quad \text{by \eqref{eq:cond23}}.
\end{align*}
By Weyl's inequality, it holds that  $\|\mU(\tau+1)\|_2\leq \bar\rho_u$,
and $\|\va(\tau+1)\|_2\leq \bar\rho_a$.
% \begin{align*}
%      &\left\|\mU(\tau+1)-\mU(0)\right\|_F \leq\sum_{s =0 }^{\tau}\left\|\mU(s+1)-\mU(s)\right\|_F \\
%      &=\sum_{s =0 }^{\tau}\eta \left\|\nabla_\mU \Phi(s)\right\|_F\leq c_u\left\|\mX\right\|_F \sum_{s=0}^{\tau}\left\|\hat\vy (s) - \vy\right\|_2\\
%      &\leq  \eta c_w\left\|\mX\right\|_F\sum_{s=0}^{\tau} \left(1- \eta \frac{\lambda_0}{2}\right)^{s/2}\|\hat \vy(0)-\vy\|_2\leq \frac{4}{\lambda_0} 
%       c_u\left\|\mX\right\|_F\|\hat \vy(0)-\vy\|_2
% \end{align*}
% \begin{align*}
%      &\left\|\va(\tau+1)-\va(0)\right\|_F \leq\sum_{s =0 }^{\tau}\left\|\va(s+1)-\va(s)\right\|_F \\
%      &=\sum_{s =0 }^{\tau}\eta \left\|\nabla_\va \Phi(s)\right\|_F\leq c_a\left\|\mX\right\|_F \sum_{s=0}^{\tau}\left\|\hat\vy (s) - \vy\right\|_2\\
%      &\leq  \eta c_w\left\|\mX\right\|_F\sum_{s=0}^{\tau} \left(1- \eta \frac{\lambda_0}{2}\right)^{s/2}\|\hat \vy(0)-\vy\|_2\leq \frac{4}{\lambda_0} 
%       c_a\left\|\mX\right\|_F\|\hat \vy(0)-\vy\|_2
% \end{align*}

(2) Next,  using \eqref{eq:deltaz}, we have
\begin{align*}
    &\left\|\mZ(\tau+1)-\mZ(0)\right\|_F\\  
    \leq&\bar\rho_a^{-1} \left(c_w\left\|\mW(\tau+1)-\mW(0)\right\|_2 + c_u\left\|\mU(\tau+1 )-\mU(0)\right\|_2\right) \left \|\mX\right\|_F\\
    \leq& \frac{4}{\lambda_0} \bar\rho_a^{-1}\left(
      c_w^2+c_u^2\right)\left\|\mX\right\|_F^2\|\hat \vy(0)-\vy\|_2\\
    \leq& \frac{2-\sqrt{2}}{2}\sqrt{\lambda_0}, \quad \text{by \eqref{eq:cond22}}
\end{align*}
By Wely's inequality, it implies that $\sigma_{\min}\left(\mZ(\tau+1)\right) \geq \sqrt{\frac{\lambda_0}{2}}$. Thus, it holds $\lambda_{\tau+1} \geq \frac{\lambda_0}{2}$.

(3) Furthermore, we define $\vg \triangleq \va(\tau+1)^\top\mZ(\tau)$ and  note that 
\begin{align*}
    &\Phi(\tau+1) 
    -\Phi(\tau) \\ =&\frac{1}{2}\left\|\hat\vy(\tau+1)-\hat\vy(\tau)\right\|_2^2
    +\left(\hat\vy(\tau+1)-\vg\right)^\top\left(\hat\vy(\tau)-\vy\right)
    + \left(\vg-\hat\vy(\tau))\right)^\top (\hat\vy(\tau)-\vy).
\end{align*}
We bound each term of the RHS of this equation individually.
Firstly, using \eqref{eq:deltay}, we have 
\begin{align*}
    &\left\|\hat\vy(\tau+1)-\hat\vy(\tau)\right\|_2\\
    \leq& \left(c_w\left\|\mW(\tau+1)-\mW(\tau)\right\|_2 + c_u\left\|\mU(\tau+1)-\mU(\tau)\right\|_2+c_a\|\va(\tau+1)-\va(\tau)\|_2\right) \left\|\mX\right\|_F\\
    \leq& \eta \cdot C_1 \left\|\hat\vy(\tau)-\vy\right\|_2,
\end{align*}
where $C_1 \triangleq \left(c_w^2+c_u^2+ c_a^2\right)\left\|\mX\right\|_F^2$.

Secondly, by  \eqref{eq:deltaz}, we have
\begin{align*}
   &\left(\hat\vy(\tau+1)-\vg\right)^\top\left(\hat\vy(\tau)-\vy\right)\\
    \leq& \left\|\va(\tau+1)\right\|_2\left\|\mZ(\tau+1)-\mZ(\tau)\right\|_F\left\|\hat\vy(\tau)-\vy\right\|_2\\
    \leq&  \left(c_w\left\|\mW(\tau+1)-\mW(\tau)\right\|_2 + c_u\left\|\mU(\tau+1)-\mU(\tau)\right\|_2\right) \left \|\mX\right\|_F\left\|\hat\vy(\tau)-\vy\right\|_2\\
    \leq& \eta\left(c_w^2+c_u^2\right)\left\|\mX\right\|_F^2\left\|\hat\vy(\tau)-\vy\right\|_2^2\\
    \leq&  \eta\cdot C_2\left\|\hat\vy(\tau)-\vy\right\|_2^2, 
\end{align*}
where $ C_2 \triangleq \left(c_w^2+c_u^2\right)\left\|\mX\right\|_F^2$.

Lastly, using the fact $\left(\va(\tau+1)-\va(\tau)\right)^\top=-\eta \nabla_\va \Phi(\tau) $, we have
\begin{align*}
    &\left(\vg-\hat\vy(\tau)\right)^\top (\hat\vy(\tau)-\vy))\\
=&-\eta \left(\nabla_\va \Phi(\tau) \mZ(\tau)\right)^\top (\hat\vy(\tau)-\vy)\\
=& -\eta(\hat\vy(\tau)-\vy)^\top \mZ(\tau)^\top \mZ(\tau)(\hat\vy(\tau)-\vy)\\
\leq& -\eta \frac{\lambda_0}{2}\|\hat\vy(\tau)-\vy\|_2^2,
\end{align*}
where we use the induce assumption $\lambda_\tau > \frac{\lambda_0}{2}$.

Putting all bounds together, we have 
\begin{align*}
    \Phi({\tau+1})
    &=\left(1-\eta ( \lambda_0-\eta C_1^2-2C_2) \right)\Phi({\tau})\\
    &\leq \left(1-\eta ( \lambda_0-4C_2) \right)\Phi({\tau}),\quad \text{by the condition on $\eta$}\\
    &\leq \left(1-\eta\frac{\lambda_0}{2}\right)\Phi({\tau}), \quad \text{by~\eqref{eq:cond23}}.
\end{align*}
\end{proof}

\section{PROOF FOR SECTION~\ref{sec:inital1}}\label{prof:thm3}
\subsection{Proof of Lemma~\ref{lem:reluk}}
\begin{proof}[Proof of Lemma~\ref{lem:reluk}] By~\eqref{eq:defkraw}, it is easy to show that  for all $i,j \in[n]$ and $l\geq1$,
	\begin{align*}
		\mK_{ii}^{(l)}= \mK_{jj}^{(l)}, \quad \mK_{ij}^{(l)}= \mK_{ji}^{(l)}.
	\end{align*}
Recall that	we define $\cos\theta_{ij}^{(l)} = \frac{\sigma_w^2\mK_{ij}^{(l-1)}+d^{-1}\vx_i^\top\vx_j}{\sigma_w^2\mK_{ij}^{(l-1)}+1}$ and it holds that
	\begin{align*}
	\boldsymbol{\Lambda}_{ij}^{(l)} =&  \left[\begin{array}{cc}
	\sigma_w^2\mK_{ii}^{(l-1)} +1 &  \sigma_w^2\mK_{ij}^{(l-1)} +\frac{1}{d}\vx_i^\top\vx_j\\ 
	\sigma_w^2\mK_{ji}^{(l-1)} + \frac{1}{d}\vx_j^\top\vx_i &  \sigma_w^2\mK_{ii}^{(l-1)} +1
	\end{array}\right]\\
	= &\left(\sigma_w^2\mK_{ii}^{(l-1)} +1\right)\left[\begin{array}{cc}
	1 & \cos\theta_{ij}^{(l)}\\ 
	\cos\theta_{ij}^{(l)} & 1
	\end{array}\right]\\
	=&\rho^{(l)}\left[\begin{array}{cc}
	1 & \cos\theta_{ij}^{(l)}\\ 
	\cos\theta_{ij}^{(l)} & 1
	\end{array}\right].
	\end{align*}
	For $i=j$, 	$\boldsymbol{\Lambda}_{ij}^{(l)} = \left(\sigma_w^2\mK_{ii}^{(l-1)} +1\right)\left[\begin{array}{cc}
	1 & 1\\ 
	1 & 1
	\end{array}\right]$.	
	By the homogeneity of ReLU, we have 
	\begin{align*}
		\mK_{ii}^{(l)} &= 2 \E_{(\ru,\rv)^\top\sim \gN(0, \boldsymbol{\Lambda}_{ii}^{(l)})}\left[\phi(\ru)\phi(\rv)\right]\\
		&=2 \left(\sigma_w^2\mK_{ii}^{(l-1)} +1\right)\E_{(\ru',\rv')^\top\sim \gN\left(0, \left[\begin{array}{cc}
			1 & 1\\ 
		     1 & 1
			\end{array}\right]\right)}\left[\phi(\ru')\phi(\rv')\right]\\
		&=\left(\sigma_w^2\mK_{ii}^{(l-1)}+1\right) \cdot Q(1) \\
		&= \sigma_w^2\mK_{ii}^{(l-1)}+1,
	\end{align*}
	Note that $\mK_{ii}^{(0)}=0$, and it is easy to show that for all $i \in[n]$ and $l\geq1$, it holds 
	\begin{align*}
	\rho^{(l)}=\mK_{ii}^{(l)} = \frac{1-\sigma_w^{2l}}{1-\sigma_w^2}.
	\end{align*}
For all $(i,j)\in[n] \times [n]$, we have 
\begin{align*}
\mK_{ij}^{(l)} &= 2 \E_{(\ru,\rv)^\top\sim \gN(0, \boldsymbol{\Lambda}_{ii}^{(l)})}\left[\phi(\ru)\phi(\rv)\right]\\
&=2 \left(\sigma_w^2\mK_{ii}^{(l-1)} +1\right)\E_{(\ru',\rv')^\top\sim \gN\left(0, \left[\begin{array}{cc}
	1 & \cos\theta_{ij}^{(l)}\\ 
	\cos\theta_{ij}^{(l)} & 1
	\end{array}\right]\right)}\left[\phi(\ru')\phi(\rv')\right]\\
&=\left(\sigma_w^2\mK_{ii}^{(l-1)}+1\right) \cdot Q\left(\cos\theta_{ij}^{(l)}\right) \\
&=\rho^{(l)} Q\left(\cos\theta_{ij}^{(l)}\right).
\end{align*}
Consequently, we prove~\eqref{eq:Kdef}.

By substituting~\eqref{eq:Kdef} into the definition of $\cos\theta_{ij}^{(l)}$, one can show that
\begin{align*}
\cos\theta_{ij}^{(l)} = \frac{\sigma_w^2\mK_{ij}^{(l-1)}+\frac{1}{d}\vx_i^\top\vx_j}{\sigma_w^2\mK_{ij}^{(l-1)}+1}
=\frac{\left(\mK_{ij}^{(l-1)}-1\right)Q\left(\cos\theta_{ij}^{(l-1)}\right) + \frac{1}{d}\vx_i^\top\vx_j}{\mK_{ij}^{(l-1)}}.
\end{align*}
Therefore, we have 
\begin{align*}
	\cos\theta_{ij}^{(l)}= \frac{\left( \rho^{(l)}-1\right)Q\left( \cos\theta_{ij}^{(l-1)}\right)+\frac{1}{d}\vx_i^\top\vx_j}{ \rho^{(l)}}.
\end{align*}
Letting $l\rightarrow \infty$, ~\eqref{eq:Kinf} is proved.
\end{proof}
\subsection{Proof of Theorem~\ref{thm:K-KL}}
\begin{proof}[Proof of Theorem~\ref{thm:K-KL}]
(\romannumeral1) About $\left\|\mK-\mK^{(l)}\right\|_F$.

 By the triangle inequality, we have
\begin{align*}
    &\left|\mK_{ij}^{(l+1)}-\mK_{ij}^{(l)}\right|\\ 
   =& \left|\rho^{(l+1)}Q\left(\cos\theta_{ij}^{(l+1)}\right)-\rho^{(l)}Q\left(\cos\theta_{ij}^{(l)}\right)\right|\\
    \leq& \left|\rho^{(l+1)}Q\left(\cos\theta_{ij}^{(l+1)}\right)-\rho^{(l+1)}Q\left(\cos\theta_{ij}^{(l)}\right)\right| + \left|\rho^{(l+1)}Q\left(\cos\theta_{ij}^{(l)}\right)-\rho^{(l)}Q\left(\cos\theta_{ij}^{(l)}\right) \right|.
\end{align*}

We bound each term individually. 

Firstly, using the fact that $|Q'(x)|\leq1$,  we have
\begin{align*}
    &\left|\rho^{(l+1)}Q\left(\cos\theta_{ij}^{(l+1)}\right)-\rho^{(l+1)}Q\left(\cos\theta_{ij}^{(l)}\right)\right|\\
    \leq& \left|\rho^{(l+1)}\cos\theta_{ij}^{(l+1)} -\rho^{(l+1)}\cos\theta_{ij}^{(l)} \right| \\
    \leq& \left|\rho^{(l+1)}\cos\theta_{ij}^{(l+1)}-\rho^{(l)}\cos\theta_{ij}^{(l)}\right|   + \left|\rho^{(l)}\cos\theta_{ij}^{(l)}-\rho^{(l+1)}\cos\theta_{ij}^{(l)}\right| \\
    \leq&\left|\rho^{(l+1)}\cos\theta_{ij}^{(l+1)}-\rho^{(l)}\cos\theta_{ij}^{(l)}\right|   + \left|\rho^{(l)}-\rho^{(l+1)}\right| \\
    =& \left|\sigma_w^2\mK_{ij}^{(l)}+\frac{1}{d}\vx_i^\top\vx_j-\left(\sigma_w^2\mK_{ij}^{(l-1)}+\frac{1}{d}\vx_i^\top\vx_j\right)\right| + \left|\frac{1-\sigma_w^{2l}}{1-\sigma_w^2}- \frac{1-\sigma_w^{2(l+1)}}{1-\sigma_w^2}\right|\\
    =&\sigma_w^2\left|\mK_{ij}^{(l)}-\mK_{ij}^{(l-1)}\right| + \sigma_w^{2l},
\end{align*}
where the first equality follows from the fact that $\rho^{(l+1)}= \sigma_w^2\mK_{ii}^{(l)} + 1$, and $\cos\theta_{ij}^{(l+1)} = \frac{\sigma_w^2\mK_{ij}^{(l)} +d^{-1}\vx_i^\top\vx_j }{\sigma_w^2\mK_{ii}^{(l)} + 1}$.

Secondly, using the fact that  $|Q(x)|\leq1$, we have
\begin{align*}
    \left|\rho^{(l+1)}Q\left(\cos\theta_{ij}^{(l)}\right)-\rho^{(l)}Q\left(\cos\theta_{ij}^{(l)}\right) \right|\leq \left|\rho^{(l+1)}-\rho^{(l)} \right|=\sigma_w^{2l}.
\end{align*}

Consequently, for $l\geq 1$, it holds that 
\begin{align*}
    \left|\mK_{ij}^{(l+1)}-\mK_{ij}^{(l)}\right|\leq \sigma_w^2\left|\mK_{ij}^{(l)}-\mK_{ij}^{(l-1)}\right| + 2\sigma_w^{2l}.
\end{align*}
This  implies that, for $l\geq 1$, we have 
\begin{align*}
    \left|\mK_{ij}^{l}-\mK_{ij}^{(l-1)}\right| \leq (2l-1)\sigma_w^{2(l-1)}.
\end{align*}
Therefore, it holds that
\begin{align*}
    \left|\mK_{ij}-\mK_{ij}^{(l)}\right|=\order{l\sigma_w^{2l}},
\end{align*}
which implies that
\begin{align*}
    \left\|\mK-\mK^{(l)}\right\|_F=\order{n\sigma_w^{2l}l}.
\end{align*}

(\romannumeral2) About the positive definiteness of $\mK$. 

The proof of this part is similar with those of~\citet{oymak2020toward,nguyen2020global} which are based on Hermite polynomials. We refer the reader to~\citet{NIPS2016_abea47ba} for a detailed introduction about Hermite polynomials.

Following from Lemma~\ref{lem:herrelu}, for $(i,j)\in[n] \times [n]$, it holds that 
\begin{align*}
	\mK_{ij} = \frac{1}{1-\sigma_w^2}Q(\cos\theta_{ij}) =   \frac{1}{1-\sigma_w^2} \sum_{r=0}^{\infty} \mu_r^2(\phi)(\cos{\theta_{ij}})^{r}.
\end{align*}
Let $\mH=[\vh_1,\cdots,\vh_n]$ where $\vh_1,\cdots,\vh_n$ be  unit vectors such that $\cos\theta_{ij}=\vh_i^\top\vh_j$  for all $(i,j)\in[n] \times [n]$. It is easy to check that $[(\mH^\top\mH)^{(\odot r)}]_{ij} =(\vh_i^\top\vh_j)^{r}$ holds for all $(i,j)\in[n] \times [n]$. Then, $\mK$ can be rewritten as
\begin{equation}
\mK = \frac{1}{1-\sigma_w^2}\sum_{r=0}^{\infty}\mu_r^2(\phi)(\mH^\top\mH)^{(\odot r)}.
\end{equation}
 Following from Lemma~\ref{lem:hadma}, we show that $\mK$ is a sum of a series of PSD matrices.  Thus, it suffices to show that $\mK$ is strictly positive definite if there exists a $r$ such that  $\mu_r^2(\phi)\neq0$ and $(\mH^\top\mH)^{(\odot r)}$  is strictly positive definite.

For any unit vector $\vv = [v_1,\cdots,v_n]^\top \in R^{n}$, it holds that
\begin{align*}
\vv^\top (\mH^\top\mH)^{(\odot r)} \vv  =& \sum_{i,j}v_iv_j(\vh_i^\top\vh_j)^{r}\\
=& \sum_{i,j}v_iv_j(\cos{\theta_{ij}})^{r}\\
=& \sum_{i} v_i^2 +\sum_{i\neq j} v_iv_j(\cos{\theta_{ij}})^{r}\\
=& 1 +\sum_{i\neq j} v_iv_j(\cos{\theta_{ij}})^{r}
\end{align*}

 Let us define $\beta = \max_{i\neq j}|\cos\theta_{ij}|$. By~\eqref{eq:Kinf}, it holds that 
\begin{align*}
	\left|\cos\theta_{ij}\right| =& \left|\sigma_w^2Q(\cos\theta_{ij})+(1-\sigma_w^2)\frac{1}{d}\vx_i^\top\vx_j\right|\\
	=&\left|\sigma_w^2Q(\cos\theta_{ij})\right|+\left|(1-\sigma_w^2)\frac{1}{d}\vx_i^\top\vx_j\right|\\
	< &\sigma_w^2+ 1-\sigma_w^2 =1.
\end{align*}
for all $(i,j)\in[n] \times [n]$. The last inequality follows from that facts that $|Q(x)|\leq 1$ for $|x|\leq1$ and $\max_{i\neq j}\left(\left|\frac{1}{d}\vx_i^\top\vx_j\right|\right)<1$ (by Assumption~\ref{assum:data}). Therefore, it holds that 
\begin{align*}
\beta<1.
\end{align*}
Taking $r>-\frac{\log(n)}{\log(\beta)}$, we have 
\begin{align*}
	\left|\sum_{i\neq j}v_iv_j(\cos{\theta_{ij}})^{r}\right|\leq \sum_{i\neq j}|v_i||v_j|\beta^r
	\leq \left(\sum_{i}\left|v_i\right|\right)^2\beta^r\leq n\beta^r<1.
\end{align*}
By Weyl's inequality, it holds that $\vv^\top (\mH^\top\mH)^{(\odot r)} \vv>0$, \emph{i.e.} $(\mH^\top\mH)^{(\odot r)}$ is positive definite. Following from Lemma~\ref{lem:herrelu},  it holds that $\mu_r^2(\phi)>0$. Therefore, the positive definiteness of $\mK$ is proved.
\end{proof}

\section{PROOF FOR SECTION~\ref{sec:inital2}}\label{prof:thm4}
\subsection{Proof of Theorem~\ref{thm:iwt2fwt}}
\begin{proof}
 Using standard bounds on the operator norm of Gaussian matrices, it holds \emph{w.p.} $\geq 1-\exp\left(-m\right)$, 
 \begin{align*}
     \left\|\vz_i^{(l+1)}-\vz_i^{(l)}\right\|_2\leq  2\sqrt{2}\sigma_w\left\|\vz_i^{(l)}-\vz_i^{(l-1)}\right\|_2, 
 \end{align*}
Therefore, it holds that
\begin{align*}
    \left\|\vz_i^{(l)}-\vz_i^{(l-1)}\right\|_2 = \order{\left\|\vz_i^{(1)}\right\|_2},
\end{align*}
and 
\begin{align*}
\left\|\vz_i-\vz_i^{(l)}\right\|_2=\order{ \left(2\sqrt{2}\sigma_w\right)^l\left\|\vz_i^{(1)}\right\|_2}.
\end{align*}

For $\vz_i^{(1)}$, we have 
\begin{align*}
    \E\left[\frac{1}{m}\left(\vz_i^{(1)}\right)^\top \vz_i^{(1)}\right] = \E\left[\frac{1}{m}\phi(\mU\vx_i)^\top\phi(\mU\vx_i)\right]=1.
\end{align*}
Using  Bernstein inequality, it holds \emph{w.p.} $\geq 1-\exp{-\Omega\left(mt^2\right)}$
\begin{align*}
    \left|\frac{1}{m}\left(\vz_i^{(1)}\right)^\top \vz_i^{(1)}-1\right|\leq t.
\end{align*}
Consequently, we have 
\begin{align*}
    \left|\mG_{ij}-\mG_{ij}^{(l)}    \right| =& \left|\vz_i^\top\vz_j-\left(\vz_i^{(l)}\right)^\top\left(\vz_j^{(l)}\right)\right|\\
    \leq &\left|\vz_i^\top\vz_j-\vz_i^\top\vz_j^{(l)}\right| + \left|\vz_i^\top\vz_j^{(l)}-\left(\vz_i^{(l)}\right)^\top\left(\vz_j^{(l)}\right)\right|\\
    \leq&\left\|\vz_i\right\|_2 \left\|\vz_j-\vz_j^{(l)}\right\|_2 + \left\|\vz_i^{(l)}\right\|_2 \left\|\vz_i-\vz_i^{(l)}\right\|_2 \\
    \leq& C \left(2\sqrt{2}\sigma_w\right)^L m\left(1 + \sqrt{t}\right).
\end{align*}
where $C$ is an absolute positive constant. Lastly,
letting $t$ be an absolute positive constant, we prove Theorem~\ref{thm:iwt2fwt} by applying the simple union bound.
\end{proof}

\section{PROOF FOR SECTION~\ref{sec:inital3}}~\label{prof:thm5}
In this section, we define $\hat{\mG}_{ij}^{(l)} = \E_{\rvw\sim\gN(0,\mI)}[\phi(\rvw^\top\vh) \phi(\rvw^\top\vh')]$.
Combining Lemma~\ref{lem:herrelu} and the homogeneity of ReLU, we write $\hat{\mG}_{ij}^{(l)}$ as
\begin{align*}
&\hat\mA^{(l)}_{ij}=\vh^\top\vh'  \\
&\cos\hat{\theta}_{ij}^{(l)} = \frac{\hat\mA^{(l)}_{ij}}{\sqrt{\hat\mA^{(l)}_{ii}\hat\mA^{(l)}_{jj}}}\\
&\hat{\mG}_{ij}^{(l)} =\sqrt{\hat\mA^{(l)}_{ii}\hat\mA^{(l)}_{jj}} Q(\cos\hat{\theta}_{ij}^{(l)})
\end{align*}
By the triangle inequality, we have
\begin{align}
\left|\frac{1}{m}\mG^{(l)}_{ij} - \mK^{(l) }_{ij}\right|\leq  \left|\frac{1}{m}\mG^{(l)}_{ij}  - \hat{\mG}^{(l) }_{ij}\right|+  \left|\hat{\mG}^{(l)}_{ij} - \mK^{(l) }_{ij}\right|.
\label{eq:tG-K}
\end{align}
\subsection{Proof of Theorem~\ref{thm:G-K}}
\begin{lem}\label{G-KII}
	 For $i=j$,
with probability at least $1-l\exp{-\Omega(m\varepsilon^2)+\order{\frac{1}{\varepsilon}}}$, it holds that
	\begin{align}
	\left|\frac{1}{m}\mG_{ii}^{(l)} -\mK_{ii}^{(l)}\right|\leq \varepsilon,
	\end{align}
or equivalently, $\left|\frac{1}{m}\mG_{ii}^{(l)} -\rho^{(l)}\right|\leq \varepsilon$.
	\label{lem:lem8}
\end{lem}
\begin{proof}
	Following Lemma~\ref{lem:rec}, we reconstruct $\mG_{ii}^{(l+1)}$ as
	\begin{align*}
	\mG_{ii}^{(l+1)} = \phi\left( \mM \vh\right)^\top\phi\left( \mM \vh\right),
	\end{align*}
	where $\|\vh\|_2^2=\vh^\top\vh=\frac{\sigma_w^2}{m}\mG_{ii}^{(l)} + 1$. 
	
	(1) For \emph{fixed} $\vh$, by the standard Bernstein inequality, it holds \emph{w.p.} $\geq 1-\exp{-\Omega\left(m\varepsilon^2\right)}$,
	\begin{align*}
	\left|\frac{1}{m}\mG_{ii}^{(l+1)}-\hat\mG_{ii}^{(l+1)}\right|\leq \varepsilon.
	\end{align*}
	
	(2) For \emph{all} $\vh$, note that the $\varepsilon$-net size is at most $\exp{\order{l\log \frac{1}{\varepsilon}}}$. Therefore, it holds \emph{w.p.} $\geq 1-\exp{-\Omega(m\varepsilon^2)+\order{l\log \frac{1}{\varepsilon}}}$,    
	\begin{align*}
	\left|\frac{1}{m}\mG_{ii}^{(l+1)}-\hat\mG_{ii}^{(l+1)}\right|\leq \varepsilon.
	\end{align*}
	
	(3) Substitute the choice of $\vh$ such that $\vh^\top\vh=\frac{\sigma_w^2}{m}\mG_{ii}^{(l)} + 1$. We have
	\begin{align*}
	\left|\hat\mG_{ii}^{(l+1)}- \mK_{ii}^{(l+1)} \right|= \sigma_w^2 \left|\frac{1}{m}\mG_{ii}^{(l)} - \mK_{ii}^{(l)}\right|.
	\end{align*}
	And we have, \emph{w.p.} $\geq1-\exp{-\Omega(m\varepsilon^2)+\order{l\log \frac{1}{\varepsilon}}}$, 
	\begin{align*}
	\left|\frac{1}{m}\mG_{ii}^{(l+1)}-\mK_{ii}^{(l+1)}\right|\leq\left|\frac{1}{m}\mG_{ii}^{(l+1)}-\hat\mG_{ii}^{(l+1)}\right|+\left|\hat\mG_{ii}^{(l+1)}- \mK_{ii}^{(l+1)} \right|\leq \sigma_w^2 \left|\frac{1}{m}\mG_{ii}^{(l)} - \mK_{ii}^{(l)}\right|+\varepsilon
	\end{align*}
	which implies that with probability at least $1-l\exp{-\Omega(m\varepsilon^2)+\order{l\log \frac{1}{\varepsilon}}}$, we have 
	\begin{align}
	\left|\mG_{ii}^{(l)} -\mK_{ii}^{(l)}\right|\leq \frac{1-\sigma_w^{2l}}{1-\sigma_w^2}\varepsilon.
	\end{align}
\end{proof}
\begin{lem} \label{G-KIJ}
	For $i\neq j$,
	with probability at least
	$1-l^2\exp{-\Omega(m\varepsilon^2)+\order{l\log \frac{1}{\varepsilon}}}$,  it holds that
	\begin{align*}
	\left|\frac{1}{m}\mG_{ij}^{(l)}-\mK_{ij}^{(l)}\right|\leq \varepsilon.
	\end{align*}
\end{lem}
\begin{proof}
	Following Lemma~\ref{lem:rec}, we reconstruct $\mG_{ij}^{(l+1)}$ as
	\begin{align*}
	\mG_{ij}^{(l+1)} =\phi\left( \mM \vh\right)^\top\phi\left( \mM \vh'\right),
	\end{align*}
	where $\vh^\top\vh' =\frac{ \sigma_w^2}{m}\mG_{ij}^{(l)} + \frac{1}{d}\vx_i^\top\vx_j$.
	
	(1) For \emph{fixed} $\vh$ and $\vh'$, by the standard Bernstein inequality, we have \emph{w.p.} $\geq 1-\exp{-\Omega(m\varepsilon^2)}$ 
	\begin{align*}
	\left|\frac{1}{m}\mG_{ij}^{(l+1)}-\hat\mG_{ij}^{(l+1)}\right|\leq \varepsilon.
	\end{align*}
	
	(2) For \emph{all} $\vh,\vh'$, note that the $\varepsilon$-net size is at most $\exp{\order{l\log \frac{1}{\varepsilon}}}$. Therefore,  \emph{w.p.} $\geq 1-\exp{-\Omega(m\varepsilon^2)+\order{l\log \frac{1}{\varepsilon}}}$,  it holds that
	\begin{align*}
	\left|\frac{1}{m}\mG_{ij}^{(l+1)}-\hat\mG_{ij}^{(l+1)}\right|\leq \varepsilon.
	\end{align*}
	
	(3) Substituting the choice of $\vh$ and $\vh'$ such that $\vh^\top\vh'=\frac{\sigma_w^2}{m}\mG_{ij}^{(l)} + \frac{1}{d}\vx_i^\top\vx_j$. We have
	\begin{align*}
	&\left|\hat{\mG}^{(l+1)}_{ij} - \mK^{(l+1) }_{ij}\right|\\ =&\left|\sqrt{\hat\mA^{(l+1)}_{ii}\hat\mA^{(l+1)}_{jj}} Q(\cos\hat{\theta}_{ij}^{(l+1)}) - \rho^{(l+1)} Q(\cos{\theta}_{ij}^{(l+1)})\right|\\
	\leq&\left|\left(\sqrt{\hat\mA^{(l+1)}_{ii} \hat\mA^{(l+1)}_{jj}}- \rho^{(l+1)}\right) Q(\cos\hat{\theta}_{ij}^{(l+1)}) \right| +\left| \rho^{(l+1)}\left(Q(\cos\hat{\theta}_{ij}^{(l+1)})- Q({\cos\theta}_{ij}^{(l+1)})\right)\right|\\
	\leq&\left|\sqrt{\hat\mA^{(l+1)}_{ii} \hat\mA^{(l+1)}_{jj}}- \rho^{(l+1)} \right| +\rho^{(l+1)}\left| \cos\hat{\theta}_{ij}^{(l+1)}- {\cos\theta}_{ij}^{(l+1)}\right|, \quad \text{ $|Q(\cdot)|<1$  and $Q(\cdot)$ is $1$-Lipschitz}\\
	\leq&\left|\sqrt{\hat\mA^{(l+1)}_{ii} \hat\mA^{(l+1)}_{jj}}- \rho^{(l+1)} \right| +\left|\left(\sqrt{\hat\mA^{(l+1)}_{ii} \hat\mA^{(l+1)}_{jj}} +\rho^{(l+1)} -\sqrt{\hat\mA^{(l+1)}_{ii} \hat\mA^{(l+1)}_{jj}}\right) {\cos\hat\theta}_{ij}^{(l+1)}-\rho^{(l+1)} {\cos\theta}_{ij}^{(l+1)}\right|\\
	\leq& 2\left|\sqrt{\hat\mA^{(l+1)}_{ii} \hat\mA^{(l+1)}_{jj}}- \rho^{(l+1)} \right| +\left|\sqrt{\hat\mA^{(l+1)}_{ii}\hat\mA^{(l+1)}_{jj}}  \cos\hat{\theta}_{ij}^{(l+1)}-\rho^{(l+1)} {\cos\theta}_{ij}^{(l+1)}\right|.\\
%	\leq& \varepsilon+  \left|\hat\mA_{ij}^{(l+1)}-\left(\sigma_w\mK_{ij}^{(l)}+1\right)\right|, \quad \text{w.p.}\geq 1-l\exp{-\Omega \left(m\varepsilon^2\right)+\order{l\log \frac{1}{\varepsilon}}} \\
%	= &\sigma_w^2\left|\frac{1}{m}\mG_{ij}^{(l)}-\mK_{ij}^{(l)}\right|+\varepsilon, \quad \text{w.p.}\geq 1-l\exp{-\Omega \left(m\varepsilon^2\right)+\order{l\log \frac{1}{\varepsilon}}}.
	\end{align*}
	From the definition of $\hat\mG^{(l)}$, it holds that $\hat\mA^{(l+1)}_{ii} =\frac{\sigma_w^2}{m} \mG_{ii}^{(l)}+1 $. Applying Lemma~\ref{lem:lem8}, it holds \emph{w.p.} $\geq 1-l\exp{-\Omega \left(m\varepsilon^2\right)+\order{l\log \frac{1}{\varepsilon}}}$, 
	\begin{align*}
		\left|\sqrt{\hat\mA^{(l+1)}_{ii} \hat\mA^{(l+1)}_{jj}}- \rho^{(l+1)} \right| = \left|\sqrt{\left(\frac{\sigma_w^2}{m} \mG_{ii}^{(l)}+1\right)\left(\frac{\sigma_w^2}{m} \mG_{jj}^{(l)}+1\right)}- \left(\sigma_w^2\mK_{ii}^{(l)} +1\right) \right|\leq \varepsilon.
	\end{align*}
	Moreover, note that $\sqrt{\hat\mA^{(l+1)}_{ii}\hat\mA^{(l+1)}_{jj}}  \cos\hat{\theta}_{ij}^{(l+1)} = \hat\mA^{(l+1)}_{ij} = \frac{\sigma_w^2}{m} \mG_{ij}^{(l)}+\frac{1}{d}\vx_i^\top\vx_j$ and  $\rho^{(l+1)} {\cos\theta}_{ij}^{(l+1)} = \sigma_w^2\mK_{ij}^{(l)}+\frac{1}{d}\vx_i^\top\vx_j$. Thus, it holds that
	\begin{align*}
		\left|\sqrt{\hat\mA^{(l+1)}_{ii}\hat\mA^{(l+1)}_{jj}}  \cos\hat{\theta}_{ij}^{(l+1)}-\rho^{(l+1)} {\cos\theta}_{ij}^{(l+1)}\right|=\sigma_w^2 \left|\frac{1}{m} \mG_{ij}^{(l)}-\mK\right|.
	\end{align*}
	Thus, \emph{w.p.} $\geq 1-l\exp{-\Omega \left(m\varepsilon^2\right)+\order{l\log \frac{1}{\varepsilon}}}$, it holds that
	\begin{align*}
	\left|\hat{\mG}^{(l+1)}_{ij} - \mK^{(l+1) }_{ij}\right|\leq \sigma_w^2\left|\frac{1}{m}\mG_{ij}^{(l)}-\mK_{ij}^{(l)}\right|+\varepsilon.
	\end{align*}
	
	Consequently, \emph{w.p.} $\geq 1-l\exp{-\Omega(m\varepsilon^2)+\order{l\log \frac{1}{\varepsilon}}}$, we have 
	\begin{align*}
	\left|\frac{1}{m}\mG_{ij}^{(l+1)}-\mK_{ij}^{(l+1)}\right|\leq\left|\frac{1}{m}\mG_{ij}^{(l+1)}-\hat\mG_{ij}^{(l+1)}\right|+ \left|\hat\mG_{ij}^{(l+1)}-\mK_{ij}^{(l+)}\right|\leq \varepsilon + \sigma_w^2\left|\frac{1}{m}\mG_{ij}^{(l)}-\mK_{ij}^{(l)}\right|.
	\end{align*}
	By applying the induction argument, one can show that for $l\geq 1$, it holds \emph{w.p.} $\geq 1-l^2\exp{-\Omega(m\varepsilon^2)+\order{l\log \frac{1}{\varepsilon}}}$, 
	\begin{align*}
	\left|\frac{1}{m}\mG_{ij}^{(l)}-\mK_{ij}^{(l)}\right|\leq \varepsilon.
	\end{align*}
	\end{proof}

Now we are ready to prove Theorem~\ref{thm:G-K}.
\begin{proof}[Proof of Theorem~\ref{thm:G-K}]
	Combing Lemmas~\ref{G-KII} and~\ref{G-KIJ} with the standard union bound, we have \emph{w.p.} $\geq1-n^2l^2\exp{-\Omega(m\varepsilon^2)+\order{l\log \frac{1}{\varepsilon}}}$
	\begin{align*}
	\left\|\frac{1}{m}\mG^{(l)}-\mK^{(l)}\right\|_F\leq n\varepsilon.
	\end{align*}
	Take $\varepsilon = \left(2\sqrt{2}\sigma_w\right)^l$ and notice that $\sigma_w^2<\frac{1}{8}$. It holds \emph{w.p.} $\geq 1-n^2l^2\exp{-\Omega(8^l\sigma_w^{2l}m)+\order{l^2}}\geq 1- n^2\exp{-\Omega(8^l\sigma_w^{2l}m)+\order{l^2}}$, 
	\begin{align*}
	\left\|\frac{1}{m}\mG^{(l)}-\mK^{(l)}\right\|_F=\order{n\left(2\sqrt{2}\sigma_w\right)^l}.
	\end{align*}
\end{proof}

% \subsection{Proof of Theorem~\ref{thm:mcondtion}}
% \begin{proof}
% Firstly, combining the triangle inequality with Theorems~\ref{thm:K-KL},~\ref{thm:iwt2fwt} and~\ref{thm:G-K}, it holds \emph{w.p.}  $\geq 1-n^2\exp{-\Omega\left(m8^l\sigma_w^{2l}\right)+\order{l^2}}$
% \begin{align*}
% 	\left\|\frac{1}{m}\mG-\mK\right\|_F\leq& \frac{1}{m}\left\|\mG-\mG^{(l)}\right\|_F + \left\|\frac{1}{m}\mG^{(l)}-\mK^{(l)}\right\|_F + \left\|\mK-\mK^{(l)}\right\|_F\\
% 	=&\order{n\left(2\sqrt{2}\sigma_w\right)^l} + \order{n\left(2\sqrt{2}\sigma_w\right)^l}+ \order{nl^{\frac{1}{2}}\sigma_w^l}\\
% 	=&\order{n\left(2\sqrt{2}\sigma_w\right)^l},
% \end{align*}
% where the last equality comes from the fact that $l^{\frac{1}{2}}\leq(2\sqrt{2})^l $, for $l\geq 1$.

% Next, we fix  $l$ to omit the explicit dependence in $l$. Specifically, we take $l=\Theta\left( \log (\lambda_*^{-1}n)/\log({\sqrt{2}}/{4\sigma_w} )\right)$. The lower bound of 
%   $l$  is large enough to ensure that $\left\|\frac{1}{m}\mG-\mK\right\|_F\leq\frac{\lambda_*}{2}$. Therefore, by Weyl's inequality, it holds that $\lambda_{0}>\frac{m}{2}\lambda_*$. Meanwhile, the upper bound of $l$  guarantees that the probability does not decrease exponentially. Thus, one can use a mild over-parameterization condition on $m$ to ensure the high probability.  

%  Consequently, one can show that, if $m = \Omega\left(\frac{n^2}{\lambda_*^2}\left( \log\frac{n}{\lambda_* t}\right)\right)$,  it holds \emph{w.p.} $\geq 1-t$
% \begin{equation*}
% \lambda_0\geq \frac{m}{2}\lambda_*.
% \end{equation*} 
% \end{proof}
\vfill
% \end{document}